\documentclass[10pt,twocolumn,letterpaper]{article}

\usepackage{iccv}
\usepackage{times}
\usepackage{epsfig}
\usepackage{graphicx}
\usepackage{amsmath}
\usepackage{amssymb}
\usepackage{booktabs}
\usepackage{amsmath,amssymb,amsfonts}
\usepackage[table,xcdraw]{xcolor}
\usepackage{multirow}
\usepackage{theorem}
\usepackage{bbding}
\usepackage{balance}
\usepackage[accsupp]{axessibility}

\newcommand{\matr}[1]{\mathbf{#1}}
\newcommand{\C}{\mathbf{C}}
\newcommand{\A}{\mathbf{A}}
\newcommand{\B}{\mathbf{\Phi}}
\newcommand{\G}{\mathbf{G}}
\newtheorem{prop}{Proposition}
\newtheorem{lemma}{Lemma}
\newtheorem{proof}{Proof}
\usepackage{float}

\usepackage[pagebackref,breaklinks,colorlinks]{hyperref}

\iccvfinalcopy 

\ificcvfinal\fi

\def\authorBlock{
    Mingze Sun\textsuperscript{1} \quad 
    Shiwei Mao\textsuperscript{1} \quad
    Puhua Jiang\textsuperscript{1, 2} \quad
    Maks Ovsjanikov\textsuperscript{3} \quad
    Ruqi Huang\textsuperscript{1}\thanks{Corresponding author: ruqihuang@sz.tsinghua.edu.cn}\\
    
    \textsuperscript{1}Tsinghua Shenzhen International Graduate School, China \\
    \textsuperscript{2}Peng Cheng Laboratory, China  \\  
    \textsuperscript{3}LIX, \'Ecole polytechnique, IP Paris,
France \\    

}

\begin{document}

\title{Spatially and Spectrally Consistent Deep Functional Maps}
\author{\authorBlock}
\maketitle
\pagestyle{empty}  
\thispagestyle{empty}

\ificcvfinal\thispagestyle{empty}\fi

\begin{abstract}

Cycle consistency has long been exploited as a powerful prior for jointly optimizing maps within a collection of shapes. In this paper, we investigate its utility in the approaches of Deep Functional Maps, which are considered state-of-the-art in non-rigid shape matching. We first justify that under certain conditions, the learned maps, when represented in the \emph{spectral} domain, are already cycle consistent. Furthermore, we identify the discrepancy that spectrally consistent maps are not necessarily \emph{spatially}, or point-wise, consistent. In light of this, we present a novel design of unsupervised Deep Functional Maps, which effectively enforces the harmony of learned maps under the spectral and the point-wise representation. By taking advantage of cycle consistency, our framework produces state-of-the-art results in mapping shapes even under significant distortions. Beyond that, by independently estimating maps in both spectral and spatial domains, our method naturally alleviates over-fitting in network training, yielding superior generalization performance and accuracy within an array of challenging tests for both near-isometric and non-isometric datasets. Codes are available at \href{https://github.com/rqhuang88/Spatially-and-Spectrally-Consistent-Deep-Functional-Maps}{https://github.com/rqhuang88/Spatially-and-Spectrally-Consistent-Deep-Functional-Maps}.

\end{abstract}

\section{Introduction}\label{sec:intro}

\begin{figure}[t!]
\centering
\includegraphics[width=9cm]{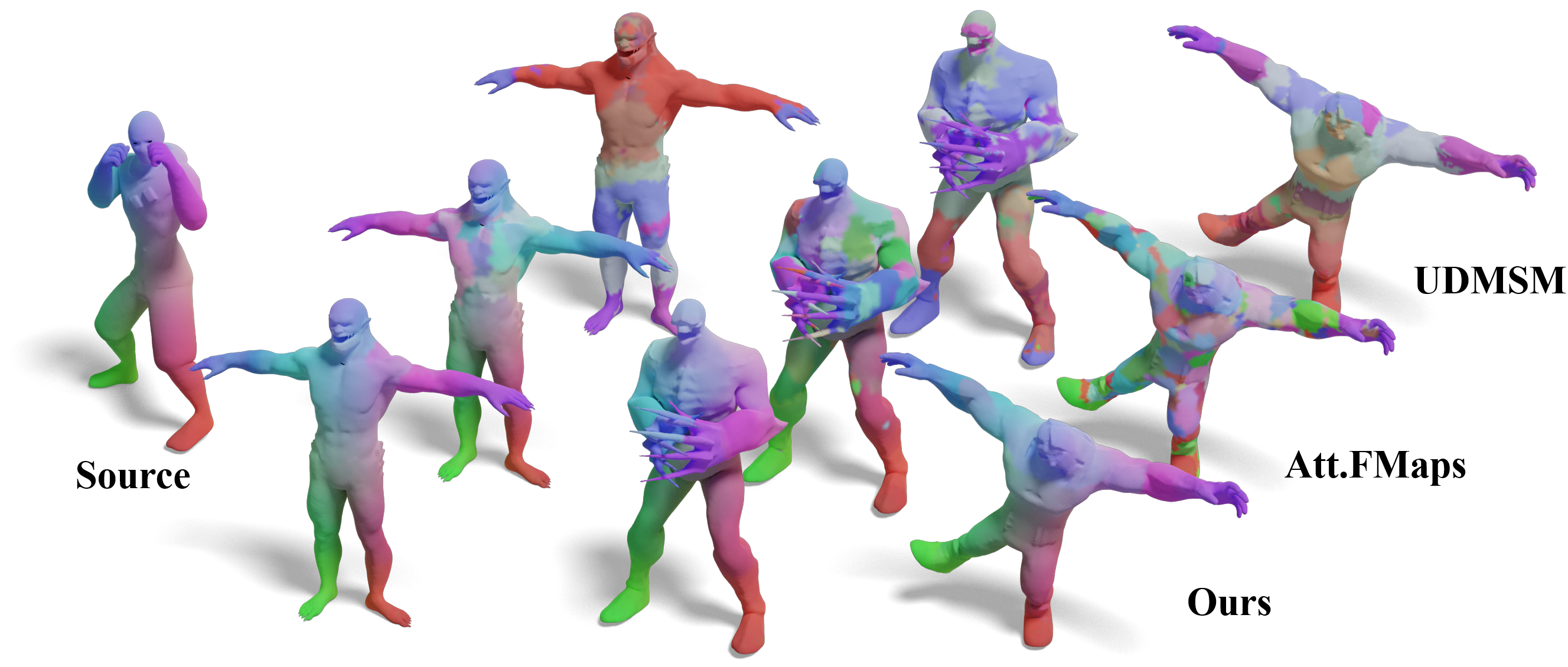}
\caption{We train unsupervised non-rigid shape matching pipelines on the challenging non-isometric dataset DT4D-H, which involves 8 categories of humanoid shapes undergoing significant distortions. Correspondences are visualized by texture transfer. Our method (bottom row) outperforms the state-of-the-art, AttentiveFMaps~\cite{li2022attentivefmaps} (middle row) and UDMSM~\cite{cao2022} (top row), by a large margin. }\label{fig:teaser}
\vspace{-1.5em}
\end{figure}

Non-rigid shape matching is a fundamental task in shape analysis, playing a pivotal role in a wide spectrum of applications including texture transfer~\cite{aigerman2015seamless}, volume parameterization~\cite{rab17}, statistical shape analysis~\cite{scape, bogo2014}, to name a few. In order to establish correspondences between shapes undergoing non-rigid deformations, early approaches~\cite{hks,wks,shot} focus on designing hand-crafted features by exploiting geometric invariance. 
Nowadays, data-driven techniques have been widely adopted to learn features in a more flexible way, leading to significant improvements regarding accuracy, efficiency, and robustness~\cite{suvey2022}. 

A noticeable trend among the learning-based shape matching approaches is based on the formalism of \textbf{Deep Functional Maps} (DFM), pioneered by the FMNet~\cite{litany2017deep}. 
Functional maps~\cite{ovsjanikov2012functional}, as a spectral map representation, allows to encode maps into compact matrices and to express desirable map priors (e.g., area-preservation, isometry, bijectivity) in simple algebraic forms. 
In particular, by learning features that optimize functional map priors, several unsupervised DFM frameworks~\cite{halimi2018self,unsuperise_fmap,CyclicFM20,smoothshells,cao2022,li2022attentivefmaps} have been proposed and, remarkably, achieve even comparable results with respect to the supervised counterparts. 
On the other hand, while the majority of DFM works focus on exploring local, pairwise-level map priors, there is relatively less attention (with the exception of~\cite{cao2022}) paid to the problem of injecting global map priors into DFM pipelines. 

As a global map prior, cycle consistency has long been used as a powerful regularizer for jointly optimizing maps among shapes, both under point-wise~\cite{huang2013consistent} and functional~\cite{funcmap14, huang2020consistent} map representation. 
In this paper, we exploit the utility of cycle consistency within DFM frameworks and propose a novel two-branch unsupervised DFM framework, which promotes cycle consistency in both spectral and spatial domains. 
We first re-examine the generic DFM framework and outline a theoretical condition, based on the residuals of a linear solver used in functional map estimation, that guarantees \emph{spectrally} cycle consistency in DFM over training shapes. 
Then we identify the discrepancy that spectrally consistent maps are not necessarily spatially, or point-wise, consistent. 
In light of this, we leverage our theoretical insight to construct a latent shape in DFM and employ this construction to formulate a novel two-branch design of unsupervised DFM. 
We align each shape's spectral embedding, represented as the eigenbasis of the Laplace-Beltrami operator, to a canonical space given by the constructed universal latent shape. 
We then compute point-wise maps in this canonical embedding domain, which promotes consistency across computed point-wise maps. 
In the end, the point-wise maps are converted to functional maps, which are expected to agree with the ones from the original spectral branch. 
Overall, our two-branch DFM takes advantage of the justified spectrally cycle consistency and further enhances it to spatially cycle consistent. 

We conduct extensive experiments on several non-rigid shape matching benchmarks, and demonstrate that our method achieves superior performance over existing methods, especially in the presence of significant distortions among shapes (see, e.g., Fig.~\ref{fig:teaser}). 
We also observe that our method, by enforcing the harmony of spatial and spectral map representations, reduces over-fitting during training, which leads to remarkable generalization performance within an array of challenging tests.  
Finally, our two-branch design can be easily plugged into any existing DFM framework, and we show evidence that it achieves significant performance gains upon both recent and early DFM approaches~\cite{unsuperise_fmap} with marginal computational burden. 

To summarize, our main contributions are as follows: (1) We perform theoretical analysis on spectrally cycle consistency of DFM frameworks; (2) We formulate a simple yet effective two-branch design of unsupervised DFM based on our theoretical justification, which introduces spatially cycle consistency. (3) We prove the effectiveness of our method through an array of challenging non-rigid shape matching tasks, demonstrating improvements upon existing methods in terms of accuracy, consistency, and generalization performance.

\section{Related Work}
\label{sec:related}
\noindent\textbf{Functional Maps} 
Our method is built upon functional map representation, introduced in~\cite{ovsjanikov2012functional} and then significantly extended in follow-up works (see, e.g., \cite{ovsjanikov2017computing}). The key idea is to encode shape correspondences as transformations between the respective spectral embeddings, which are represented by compact matrices by using reduced eigenbasis. The functional maps pipeline has been further improved in accuracy, efficiency, and robustness by many recent works including \cite{kovnatsky2013coupled,funcmap14,burghard2017embedding,rodola2017partial,commutativity}.

\noindent\textbf{Deep Functional Maps}
In contrast to axiomatic approaches that use hand-crafted features~\cite{hks, wks}, the deep functional maps approach, pioneered by FMNet~\cite{litany2017deep}, aims to \emph{learn} the optimal features from data. FMNet contributes several key designs of DFM: (1) it leverages Siamese network to conduct learning in a set of \emph{shape pairs}; (2) it advocates refining the input descriptors with \emph{non-linear} transformations. FMNet is then supervised by labeled maps to learn optimal features.

Instead of learning from labeled maps, unsupervised approaches~\cite{halimi2018self, unsuperise_fmap} demonstrate that it is sufficient to learn from geometric map priors. 
More recently, with the development of robust mesh feature extractors~\cite{diffusionNet}, more frameworks~\cite{cao2022, li2022attentivefmaps, donati2022deep,attaiki2021dpfm} are proposed to learn directly from geometry, yielding state-of-the-art performance. 

\noindent\textbf{Cycle Consistency}
Cycle consistency has long been used as a strong prior for joint map optimization among a collection of shapes. 
Axiomatic approaches detect and eliminate inconsistent cycles using consistency constraints \cite{huber2002automatic,zach2010disambiguating,nguyen2011optimization, mapcycle,CyclicFM20},
 as well as associate cycle consistency with low-rank properties of matrices encoding map networks \cite{huang2013consistent,wang2013exact,leonardos2017distributed}. 

Related to the latter, the matrix nature of the functional maps enables convenient access to map composition, which naturally bridges the functional map framework and consistent map refinement techniques \cite{wang2013,funcmap14,limitshape,shoham2019hierarchical}. Some recent learning-based approach~\cite{huang_multiway_2022} also incorporate cycle consistency in the pipeline. 
It is worth noting, though, the above works all utilize cycle consistency as a prior for \emph{test-time optimization}, which depends on test shape collection and initial maps. Contrastingly, our approach exploits cycle consistency during \emph{training} to boost feature learning and poses no constraint on a test. 

From this viewpoint, UDMSM~\cite{cao2022} is perhaps the most relevant work to ours, as both construct a latent shape during training to inject cycle consistency.  
The key difference between their work and ours, though, is how the latent shape is constructed. In~\cite{cao2022}, the authors propose to construct a  universal shape in the learned feature space and establish point-wise maps from real shapes to the universal one by training a classifier. As shown in Sec.\ref{sec:exp}, though this construction leads to great performance on matching near-isometric shapes, it suffers from large shape variability in mapping non-isometric ones. 
On the other hand, we leverage spectral information in estimating point-wise maps. As a result, our approach benefits from more direct usage of intrinsic geometric information encoded in the spectral embedding, yielding better generalization performance.

\noindent\textbf{Dual Map Representations}
Thanks to the inherent connection and efficient conversion between point-wise and functional maps, it has long been observed that jointly estimating both map representations can improve the mapping quality. For instance, in the original work~\cite{ovsjanikov2012functional}, the authors have already proposed to apply an ICP-like technique on functional maps. 
More recent advances take advantage of the multi-scale properties of the eigenbasis of the Laplace-Beltrami operator. In the works~\cite{zoomout, huang2020consistent, smoothshells, eisenberger2020deep}, conversions are done between spatial domain and a series of spectral domains spanned by eigenfunctions of increasing dimensions. This idea has also been incorporated into DFM. In AttentiveFMaps~\cite{li2022attentivefmaps}, the authors propose to fuse functional maps of different dimensions by converting them into the common spatial domain. This fusion technique in turn allows training an attention network to dynamically choose the optimal spectral resolution. 

The above methods enforce consistency across map representations by iterative projections, which is computationally heavy. By exploiting our theoretical insight on the spectral cycle consistency of DFM, we only introduce a marginal computational overhead to a standard DFM trained on a single spectral resolution. Our method is much lighter and simpler but also shows superior performance in accuracy and generalization. 

\section{Cycle Consistency of Deep Functional Maps}\label{sec:ccfmap}
In this section, we first briefly review the general deep functional maps pipeline. Then we present a theoretical analysis of its cycle consistency. 

\subsection{Deep Functional Maps}\label{sec:dfm}
We assume to be given a pair of deformable shapes $S_1$ and $S_2$, which are discretized as triangular meshes of $n_1$ and $n_2$ vertices, respectively. The generic deep functional map pipeline introduced in \cite{litany2017deep} learns a map that spectrally aligns the shapes through the following four steps:

\begin{enumerate}
	\setlength\itemsep{-0.5em}
	\item Compute the leading $k$ eigenfunctions of the Laplace-Beltrami operator on each shape, which can be treated as a high-dimensional spectral embedding of the respective shape. The eigenfunctions are stored as matrices $\B_i \in \mathbb{R}^{n_i \times k}, i = 1, 2$.
	\item Instantiate a feature extractor network, $\mathcal{F}_{\Theta}$, where $\Theta$ denotes the set of learnable parameters. By feeding forward shapes through $\mathcal{F}_{\Theta}$, descriptors are obtained and expected to be approximately preserved by the underlying map. In general, we denote $\G_i = \mathcal{F}_{\Theta}(S_i) \in \mathbb{R}^{n_i\times d}, i = 1, 2$, where $d$ is the predefined number of descriptors. They are then projected onto the eigenbasis above, resulting in a couple of coefficient matrices $\A_i =\B_i^{\dagger} \G_i\in \mathbb{R}^{k\times d}, i = 1, 2$. 
	\item Estimate the optimal functional map, ${\C}^*$, by solving the following linear system:
		\begin{equation}
			{\C}^* = {\arg\min}_{\C} E_{\mbox{desc}}(\C) + E_{\mbox{reg}}(\C),	
		\end{equation}

		where $E_{\mbox{desc}}(\C) = \Vert \C \A_1 - \A_2\Vert^2$, and $E_{\mbox{reg}}(\C)$ is the regularization term promoting structural properties of $\C$, e.g., enforcing $\C$ to be commutative with the Laplace-Beltrami operators and to be orthogonal. In particular, we let $E_{\mbox{total}}(\C) = E_{\mbox{desc}}(\C) + E_{\mbox{reg}}(\C)$.
	\item Convert the estimated functional map $\C$ to a point-wise map by conducting nearest neighbor search between the rows of $\B_1 \C$ and that of $\B_2$. 
\end{enumerate}

\begin{figure}[t!]
    \centering
    \includegraphics[width=0.47\textwidth,height=0.18\textwidth]{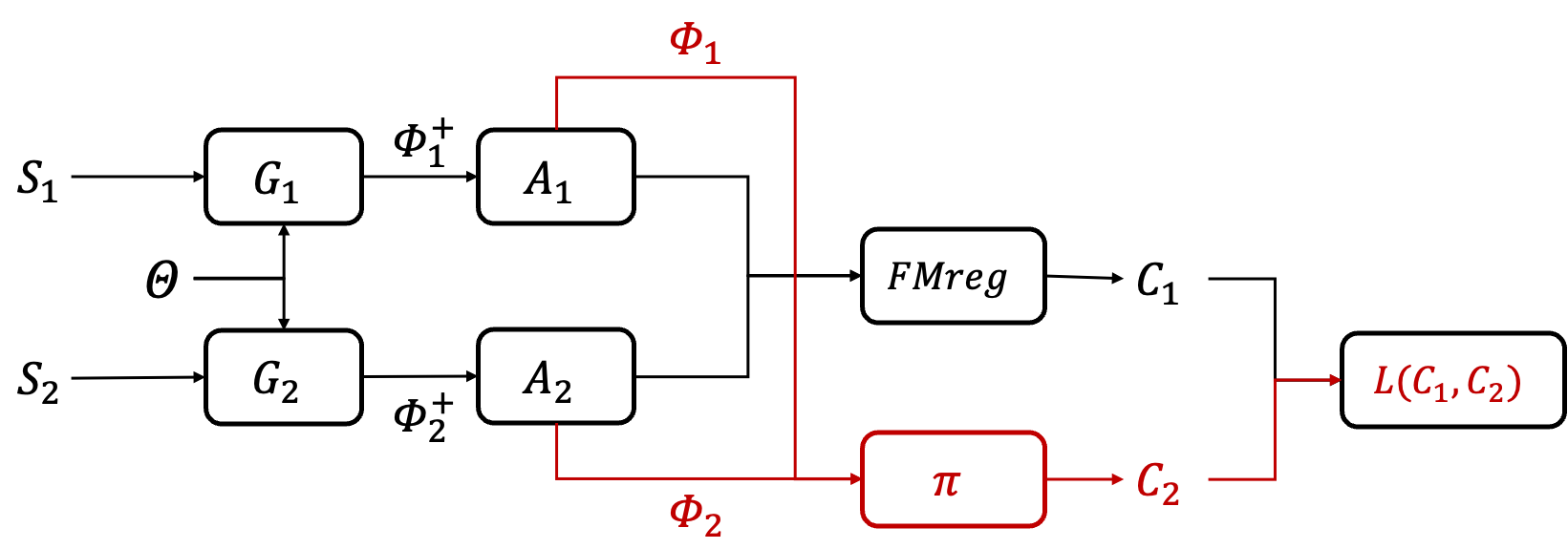}
    \caption{Overview of our two-branch DFM network. The part colored in black corresponds to a standard DFM, we introduce a novel branch, colored in red, that estimates maps from spatial perspective, our loss function is defined as Eqn.(\ref{eqn:loss}). $\B_i, \B_i^{\dagger}$ indicate projection of the regarding features into the spatial and spectral domain, respectively. See more details in Sec.~\ref{sec:net}. }\label{fig:pipeline}
    \vspace{-1.5em}
\end{figure}

The feature extractor is then trained on a set of training shape pairs, which typically enumerates all pairs from some given training set. 
The sub-optimization problem in Step (3) is generally decomposed into two parts to circumvent using an iterative solver. 
Namely, a closed-form solution is obtained either by a least-square estimation w.r.t $E_{\mbox{desc}}(\C)$~\cite{litany2017deep}, or by solving a sequence of linear systems involving Laplacian commutativity as well~\cite{donati2022deep}. Both correspond to the \emph{FMreg} layer shown in Fig.~\ref{fig:pipeline}. The rest of $E_{\mbox{reg}}(\C)$ is then set as the structural loss, corresponding to $\mathcal{L}(\C)$ in Fig.~\ref{fig:pipeline}.

\subsection{Cycle Consistency}\label{sec:theo}
We start by giving a formal definition of cycle consistency. We assume to be given a collection of shapes $\mathcal{S} = \{S_i\}_{i = 1}^n$ and a set of maps $\mathcal{T} = \{T_{ij}\}_{i, j\in [1..n]}$. We call $\mathcal{T}$ cycle consistent if for any shape $S_i$, and any closed path $(i, i_1, i_2, \cdots, i_k, i)$, the map composition along this path $T_{ii} = T_{i_k i}\circ \cdots T_{i_1 i_2} \circ T_{i i_1}$ is an identity map on $S_i$. 
We can similarly define cycle consistency on functional maps, the only difference is that map composition is given by matrix multiplication and we require the final functional map $C_{ii}$ to be an identity matrix. 

We suppose that deep functional maps are trained on $\mathcal{S}$ with respect to all possible pairs. Then the global energy is given by $E_{\mbox{total}}(\mathcal{C}) = \sum_{i, j} \Vert \C_{ij} \A_i - \A_j \Vert^2 + \sum_{i, j} E_{\mbox{reg}}(\C_{ij})$, where $\mathcal{C} = \{\C_{ij}\}_{i, j\in [1..n]}$ is the set of functional maps among training shapes.

\begin{prop}
	\label{prop:main}
	 If $E_{\mbox{total}}(\mathcal{C}) = 0$, then for any shape $S_i$, and any path $(i, i_1, i_2, \cdots, i_k, i)$, the map composition $\C_{ii}$ is cycle consistent within the functional space spanned by the columns of $\A_i$, i.e., $\C_{ii} \A_i = \A_i$. 
\end{prop} 

As a consequence of Prop.~\ref{prop:main}, when $\A_i\in \mathbb{R}^{k\times d}$ is of full row rank, then $\mathcal{C} = \{\C_{ij}\}_{i, j\in [1..n]}$ is cycle consistent. In practice, we generally set $d > k$. Moreover, during network training, $\C_{ij}$ is computed via pseudo-inverse, which implicitly assumes the full-rankness of $\A_i, \A_j$. 
We refer readers to the empirical validation in Sec.~\ref{sec:net}. 

We defer the proof of Prop.~\ref{prop:main} to Supp. Material. In fact, a similar claim has been formulated and proven in~\cite{wang2013} (see Sec.~\ref{sec:dfm} therein), but in the context of map refinement via promoting cycle consistency. 

Though being technically similar, the theoretical argument of~\cite{wang2013} and that of ours have fundamentally different implications. More specifically, the former justifies a \emph{test-time optimization} algorithm, which is used to promote cycle consistency of maps among a \emph{fixed} test shape collection. While the latter suggests that spectral cycle consistency has been ensured and further leveraged to enhance the \emph{universal} feature extractor (independent of the test data) during training in \emph{any} DFM framework following the generic pipeline presented in Sec.~\ref{sec:dfm}.

\section{Two-branch Deep Functional Maps}\label{sec:dualbranch} 

It has long been recognized, both theoretically and empirically, that optimizing purely in the spectral domain is not sufficient. As a toy example, a trivial solution attaining global optima can be constructed as follows: Suppose that we have learned a feature extractor $\mathcal{F}_{\Theta}$, which returns the respective eigenbasis transformed by a universal $\A_0$. That is, $\G_i = \B_i\A_0, \forall i$, which implies $\A_i = \B_i^{\dagger} \G_i = \A_0, \forall i$. Then we have $\C_{ij} \equiv I_k$, which exactly satisfies $E_{\mbox{desc}}(\C_{ij}) = E_{\mbox{reg}}(\C_{ij}) = 0, \forall i, j$. However, it probably induces poor point-wise maps. 

In fact, in~\cite{ovsjanikov2012functional} the authors have proposed to use an ICP-like technique to encourage the estimated functional maps to be induced by some point-wise maps. In~\cite{zoomout}, the authors propose a spectral upsampling method for map refinement, which essentially converts maps back and forth between spectral and spatial domains. Moreover, the following lemma from~\cite{zoomout} sheds light on the necessity of taking both spectral and spatial representations into consideration. 

\begin{lemma}\label{lemma:iso}
	Given a pair of shapes $S_1, S_2$ each having non-repeating Laplacian eigenvalues, which are the same. A point-wise map $ T: S_1\rightarrow S_2$ is an isometry if and only if the corresponding functional map $\C$ in the complete Laplacian basis is both diagonal and orthonormal. 
\end{lemma}

The above lemma suggests that apart from promoting the structural properties of functional maps, it is also critical to enforce them to be associated with certain point-wise maps, or termed as \emph{properness} of functional maps in~\cite{proper}.

Finally, we remark that some recent DFM advances also promote the properness of the resulting spectral maps. For instance, AttentiveFMaps~\cite{li2022attentivefmaps} follows the spirit of ZoomOut~\cite{zoomout} and explicitly performs a conversion between spectral and spatial map representations across different dimensions of eigenbasis; UDMSM~\cite{cao2022} constructs explicitly a universal shape in the feature space, and enforce the spectral map estimation to be consistent with the spatial maps induced via the universal shape.  

\subsection{Two-branch Map Estimation}\label{sec:2branch}
In this part, we leverage our observation made in Prop.~\ref{prop:main} and propose a novel, simple yet effective design of unsupervised deep functional maps, which introduces a new branch that independently estimates maps from spatial perspective. 

Our key insight is that, once cycle consistency is valid and $\A_i$ is of full row rank, $\A_i$ can be seen as a functional map from a universal \emph{latent shape}, $S_0$, to $S_i$. This perspective has been explored in several prior works~\cite{wang2013,limitshape,huang2020consistent}, we provide the following details to be self-contained. 
The above assumption implies $\C_{ij} = \A_j \A_i^{\dagger}$. Then $\C_{ij}$ can be interpreted as a functional map composition from $S_i$ to $S_0$, followed by a map from $S_0$ to $S_j$. 

On the other hand, one can align the spectral embeddings of $S_j$ to that of $S_i$ by simply transforming the former by $\C_{ij}$. 
Indeed, we convert $\C_{ij}$ into the point-wise map by the nearest neighbor searching between the rows of $\B_j \C_{ij}$ and that of $\B_i$.  
From this point of view, denoting the virtual spectral embedding of the latent shape by $\B_0$, $\B_i\A_i$ can be then treated as the spectral embedding of $S_i$ aligned to that of $S_0$. 
Therefore, given a pair of shapes $S_i, S_j$, since we have aligned their eigenbasis to the canonical frame defined by the virtual spectral embedding $\B_0$, we can align the spectral embedding of $S_i$ to $\B_0$ by computing $\B_i \A_i$. Once all the spectral embeddings are aligned to the canonical embedding domain regarding $\B_0$, we can compute the soft point-wise map between $S_i$ and  $S_j$ by nearest neighbor searching between the rows of  $\B_i \A_i$ and those of  $\B_j\A_j$. 

Based on the above derivation, given the learned features projected in the spectral domain, $\A_i, \A_j$, and a pair of indices $p \in [1..n_i], q\in [1..n_j]$, we can compute point-wise maps. Firstly we compute residual:
\begin{equation}\label{eqn:delta}
	\delta_{qp} = \Vert \B_i[p]\A_i - \B_j[q]\A_j \Vert_2,
\end{equation} 
where $\B_i[p]$ denotes the $p$-th row of $\B_i$, and similarly we define $\B_j[q]$. 
The soft point-wise map $\Pi \in \mathbb{R}^{n_j \times n_i}$ is then given by:
\begin{equation}\label{eqn:pi}
	\Pi(q, p) = \frac{\exp(-{\alpha\delta_{qp}})}{\sum_{p'} \exp(-\alpha\delta_{qp'})}.	
\end{equation}
Note that by construction, each row of $\Pi$ is non-negative and sums up to $1$, forming a probability distribution. The parameter $\alpha$ controls the entropy of each distribution -- the smaller/larger $\alpha$ is, the fuzzier/sharper the distribution is. Instead of manually tuning the optimal $\alpha$, we propose a learning scheme that dynamically controls $\alpha$ over training, which is inspired by curriculum learning~\cite{clearn}. We defer the respective details to Sec.~\ref{sec:alpha}. 

We convert the soft point-wise map to a functional map by
\begin{equation}\label{eqn:cc}
	\C_2 = \B_j^{\dagger} \Pi \B_i.
\end{equation}
In the end, we enforce $\C_2$ to be consistent with $\C_1$, the intermediate output from the \emph{FMreg} layer. 

To summarize, thanks to the spectral cycle consistency we identify in Sec.~\ref{sec:theo}, we are allowed to construct a spectral latent shape and induce spatial maps via it. By enforcing the spatial estimations to be consistent with the spectral ones, we obtain a spatially and spectrally consistent deep functional maps framework.

\begin{figure}[t!]
    \centering
    \includegraphics[width=7.5cm]{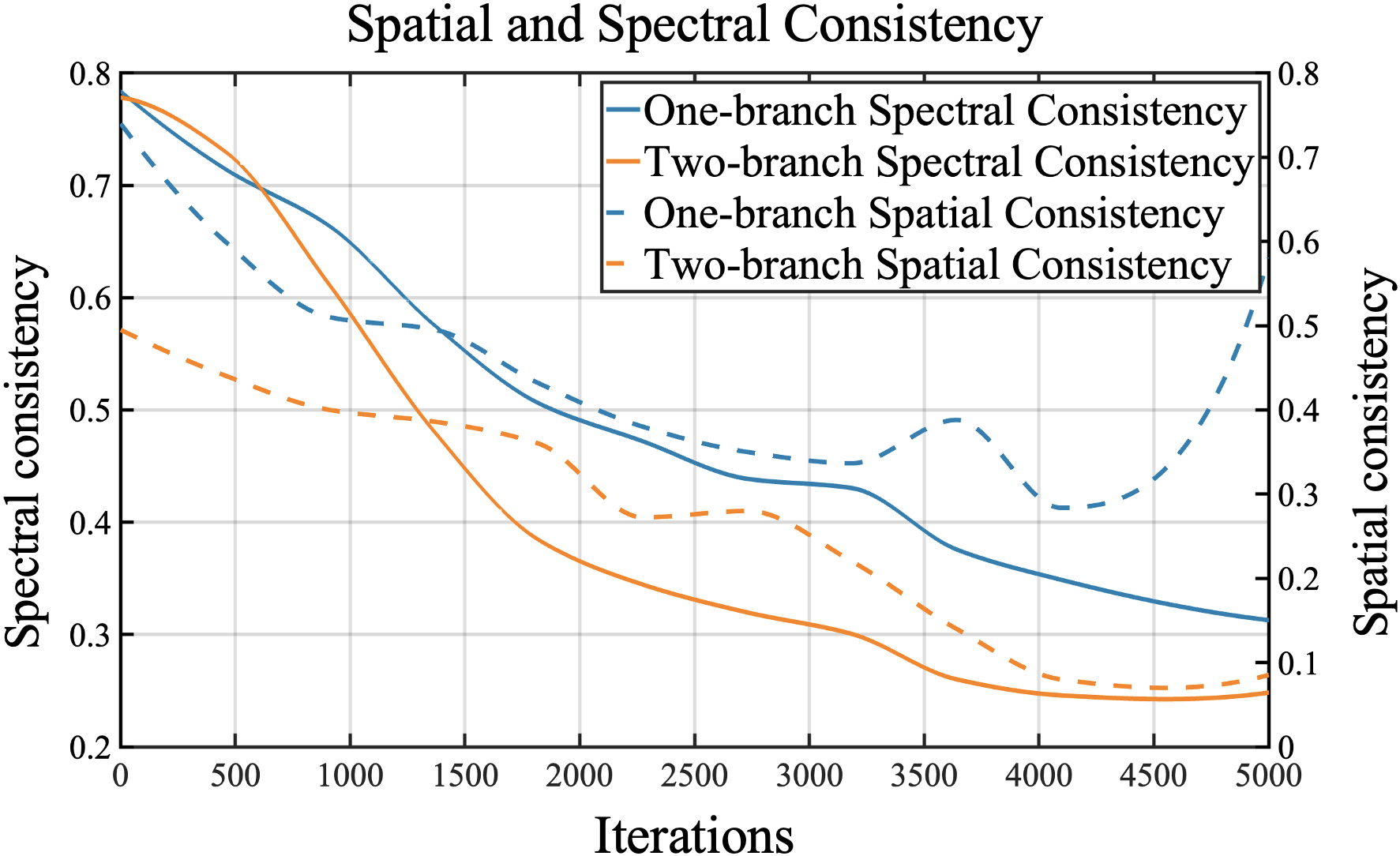}
    \caption{We train our two-branch DFM and a vanilla single-branch version on \textbf{DT4D-H} and monitor spectrally and spatially cycle consistency along the training.}
    \label{fig:theroy}
    \vspace{-1.5em}
\end{figure}

\subsection{Network Design}\label{sec:net}
As shown in Fig.~\ref{fig:pipeline}, our two-branch network is built upon a standard DFM framework. 
In the following, we denote by $\C_1$ and $\C_2$ the estimated functional maps from the original \emph{FMreg} layer and our novel branch, respectively. 

Specifically, we use DiffusionNet~\cite{diffusionNet} as our feature extractor. And WKS~\cite{wks} descriptors are fed into it as initialization of learned features. 
We borrow the \emph{FMreg} layer from~\cite{donati20}. It takes both $E_{\mbox{desc}}(\C)$ and commutativity with the Laplace-Beltrami operator into consideration, where the latter is given as:
\begin{equation}
	L_{\mbox{lap}} = \big\Vert \C_{1} {\Lambda}_1 - {\Lambda}_2 \C_{1} \big\Vert^2,
\end{equation}
where $ {\Lambda}_1$ and $ {\Lambda}_2$ are diagonal matrices of the Laplace-Beltrami eigenvalues on the two shapes.

The estimation of $\C_2$ has been described in detail in Sec.~\ref{sec:2branch}. 
In the end, we formulate the training loss as:
\begin{equation}\label{eqn:loss}
	\mathcal{L}(\C_1, \C_2) = \| \C_{1}^T \C_{1} - \matr{I}\|^2 + \Vert \C_1 - \C_2\Vert ^2,
\end{equation}
where the first term promotes the orthogonality of $\C_1$, while the second term promotes the consistency between the functional maps estimated from different branches.

Finally, we remark that by combining the \emph{FMreg} layer and $\mathcal{L}(\C_1, \C_2)$, we have incorporated every factor in Lemma~\ref{lemma:iso} into our design.

\noindent\paragraph*{Conceptual Validation}
In this part, we train a network on DT4D-H dataset (see Sec.~\ref{sec:data} for details) with our two-branch network, and a single-branch variant without our spatial map estimation branch. 
We monitor and plot the following quantities along training: (1) Average spectral cycle consistency over sampled triplets, i.e., $\frac{1}{M} \sum_{(i, j, k)} \Vert \C_{ki}\C_{jk}\C_{ij} - I \Vert^2/\Vert I \Vert^2$; (2) Average spatial cycle consistency over sampled triplets, i.e., the mean Euclidean deviation from composed maps $T_{ki} \circ T_{jk} \circ T_{ij}$ to identity map on $S_i$. Here $n = 80$ is the number of training shapes, and $M = 1000$ is the number of sampled triplets. The behavior of the blue curves after 4500 iterations verifies our argument that spectrally cycle consistency does not imply spatially cycle consistency. On the other hand, by introducing our two-branch design, the discrepancy is well compensated and evidently better cycle consistencies in both spatial and spectral domains are achieved. 

\subsection{Updating Scheme of $\alpha$ in Eqn.~(\ref{eqn:pi})}\label{sec:alpha}
The soft point-map conversion (Eqn.~\ref{eqn:pi}) has been applied in several prior works~\cite{litany2017deep, li2022attentivefmaps}, which all set $\alpha$ to be a manually selected constant. Ideally, we expect $\Pi$ in Eqn.~\eqref{eqn:pi} to be close to a permutation matrix, i.e., each row forms a binary vector. 
This seems to suggest a preference for a large $\alpha$.
Unfortunately, it would severely hinder network training, since the learned features and maps are of low quality in the early stage. On the other hand, a small $\alpha$ can alleviate such difficulty but falls short of fully pushing functional maps to be proper. As demonstrated in Sec.~\ref{sec:ablation}, neither small nor large $\alpha$ produces satisfying results. 

Based upon the above analysis, we propose a novel updating scheme, which is inspired by curriculum learning~\cite{clearn}. Namely, we initiate a small $\alpha$ at the beginning of network training and increase it by a constant step size for every fixed number of epochs. As shown in Tab.~\ref{table:1},~\ref{table:2},~\ref{table:3}, our scheme does not rely on hyperparameter tuning but also achieves state-of-the-art results. 

\subsection{Implementation Details} 
We implement our network with PyTorch~\cite{pytorch}. We use four DiffusionNet blocks~\cite{diffusionNet} as feature backbone and borrow the functional map block with Laplacian regularizer from~\cite{donati20}. The dimension of the Laplace-Beltrami eigenbasis is set to 50. WKS~\cite{wks} descriptors are used as the input signal to our network. The dimensions of the input and the output descriptors are both set to 128. During training, the value of the learning rate is set to 2e-4 with ADAM optimizer. In all experiments, we train our method for 10,000 iterations with a batch size of 1. Following the learning strategy in Sec.~\ref{sec:alpha}, we initialize $\alpha$ to $1$ and increase it by $5$ per epoch. As indicated in Eqn.~\eqref{eqn:loss}, We weigh equally the orthogonality loss with respect to $\C_1$ and the residual between $\C_1$ and $\C_2$. More implementation details are provided in the Supp. Material.

\begin{table*}[t!]
\caption{Mean geodesic errors (×100) on FAUST\_r, SCAPE\_r, and SHREC19\_r. The \textbf{best} and the \textbf{\textcolor{blue}{second best}} are highlighted. }\label{table:1}
\centering
\setlength{\tabcolsep}{11pt}
\resizebox{\textwidth}{34mm}{
\begin{tabular}{ccccccccc}
\hline
\rowcolor[HTML]{FFFFFF} 
\cellcolor[HTML]{FFFFFF}                         & Train                                                  & \multicolumn{3}{c}{\cellcolor[HTML]{FFFFFF}FAUST\_r}                                                            & \multicolumn{1}{l}{\cellcolor[HTML]{FFFFFF}} & \multicolumn{3}{c}{\cellcolor[HTML]{FFFFFF}SCAPE\_r}                                                            \\ \cline{3-5} \cline{7-9} 
\rowcolor[HTML]{FFFFFF} 
\multirow{-2}{*}{\cellcolor[HTML]{FFFFFF}Method} & Test                                                   & FAUST\_r                            & SCAPE\_r                            & SHREC19\_r                             &                                              & SCAPE\_r                            & FAUST\_r                            & SHREC19\_r                             \\ \hline
\rowcolor[HTML]{FFFFFF} 
ZM\cite{zoomout}                                                &                                                        & 6.1                                 & \textbackslash{}                    & \textbackslash{}                    &                                              & 7.5                                 & \textbackslash{}                    & \textbackslash{}                    \\
\rowcolor[HTML]{FFFFFF} 
BCICP\cite{BCICP}                                             &                                                        & 6.4                                 & \textbackslash{}                    & \textbackslash{}                    &                                              & 11.0                                & \textbackslash{}                    & \textbackslash{}                    \\

\rowcolor[HTML]{FFFFFF} 
IsoMuSh\cite{gao2021isometric}                                             &                                                        & 4.4                                 & \textbackslash{}                    & \textbackslash{}                    &                                              & 5.6                                & \textbackslash{}                    & \textbackslash{}                    \\

\rowcolor[HTML]{FFFFFF} 
Smooth Shell\cite{smoothshells}                                      &                                                        & 2.5                                 & \textbackslash{}                    & \textbackslash{}                    &                                              & 4.7                                 & \textbackslash{}                    & \textbackslash{}                    \\
\rowcolor[HTML]{FFFFFF} 
CZO\cite{huang2020consistent}                                             &                                                        & 2.2                                 & \textbackslash{}                    & \textbackslash{}                    &                                              & 2.5                                & \textbackslash{}                    & \textbackslash{}                    \\

\hline

\rowcolor[HTML]{FFFFFF} 
TransMatch\cite{trappolini2021shape}                                        & \cellcolor[HTML]{FFFFFF}                               & 2.7                                 & 33.6                                & 21.0                                &                                              & 18.3                                & 18.6                                & 38.8                                \\
\rowcolor[HTML]{FFFFFF} 
GeomFMaps\cite{donati20}                                        & \cellcolor[HTML]{FFFFFF}                               & {\textbf{\textcolor{blue}{2.6}}}                                 & {\textbf{\textcolor{blue}{3.3}}}                                & {\textbf{\textcolor{blue}{9.9}}}                                &                                              & {\textbf{\textcolor{blue}{3.0}}}                                & {\textbf{\textcolor{blue}{3.0}}}                                & \textbf{12.2}                                \\
\rowcolor[HTML]{FFFFFF} 
AttentiveFMaps\cite{li2022attentivefmaps}                                            & \multirow{-3}{*}{\cellcolor[HTML]{FFFFFF}supervised}   & \textbf{1.4}                                & \textbf{2.2}                                 & \textbf{9.4}                                 &                                              & \textbf{1.7}                                 & \textbf{1.8} & \textbf{12.2}                                \\ \hline

\rowcolor[HTML]{FFFFFF} 
NeuroMorph\cite{eisenberger2021neuromorph}                                       & \cellcolor[HTML]{FFFFFF}                               & 8.5                                 & 28.5                                & 26.3                                &                                              & 29.9                                & 18.2                                & 27.6                                \\
\rowcolor[HTML]{FFFFFF} 
SyNoRiM\cite{multi}                                       & \cellcolor[HTML]{FFFFFF}                               & 7.9                                 & 21.7                                & 25.5                                &                                              & 9.5                                & 24.6                                & 26.8                                \\

\rowcolor[HTML]{FFFFFF} 
Deep Shell\cite{eisenberger2020deep}                                        & \cellcolor[HTML]{FFFFFF}                               & {\textbf{\textcolor{blue}{1.7}}} & 5.4                                 & 27.4                                &                                              & 2.5                                 & 2.7                                 & 23.4                                \\
\rowcolor[HTML]{FFFFFF} 
AttentiveFMaps\cite{li2022attentivefmaps}                                   & \cellcolor[HTML]{FFFFFF}                               & 1.9                                 & \textbf{2.6} & 6.4 &                                              & {\textbf{\textcolor{blue}{2.2}}} & \textbf{\textcolor{blue}{2.2}}                        & 9.9                                 \\
\rowcolor[HTML]{FFFFFF} 
UDMSM\cite{cao2022}                                            & \cellcolor[HTML]{FFFFFF}                               & \textbf{1.5}                        & 7.3                                 & 21.5                                & \textbf{}                                    & \textbf{2.0}                        & 8.6                                 & 30.7                                \\
\rowcolor[HTML]{FFFFFF} 
DUO-FM\cite{donati2022deep}                                               & \cellcolor[HTML]{FFFFFF} & 2.5                                 & 4.2                                 & 6.4                                 &                                              & 2.7                                 & 2.8                                 & 8.4 \\
\rowcolor[HTML]{E7E6E6} 
\textbf{Ours}                                          & \multirow{-7}{*}{\cellcolor[HTML]{E7E6E6}unsupervised}                               & 2.3                        & \textbf{2.6}                                 & \textbf{3.8}                                 & \textbf{}                                    & 2.4                        & 2.5                                 & \textbf{4.5}                               \\
\rowcolor[HTML]{E7E6E6}  
\textbf{Ours (80 dim)}                                    & \cellcolor[HTML]{E7E6E6}                                              & {\textbf{\textcolor{blue}{1.7}}}                                 & \textbf{2.6}                        & {\textbf{\textcolor{blue}{5.5}}}                        &                                              & {\textbf{\textcolor{blue}{2.2}}}                                 & \textbf{2.0} & {\textbf{\textcolor{blue}{5.8}}}                         \\ \hline
\end{tabular}\vspace{-0.5em}}
\end{table*}
\section{Experimental Results}\label{sec:exp}

In this section, we conduct an extensive set of experiments of non-rigid shape matching on various datasets including humanoids and animals. We test on both near-isometric and non-isometric shape pairs. Our method is compared to a set of competitive baselines including axiomatic, supervised, weakly-supervised, and unsupervised learning methods. 
We emphasize that in this section, all the maps from the learning-based pipelines are directly inferred from the trained models, \textbf{without} any post-processing procedure. 
We evaluate the matching results in terms of a mean geodesic error on shapes normalized to unit area. 
Finally, our point-wise maps are all inferred by converting the output functional maps, as all the other DFM frameworks. 

\subsection{Datasets}\label{sec:data}
\noindent\textbf{FAUST\_r:} The remeshed version~\cite{BCICP} of FAUST dataset\cite{bogo2014} contains 100 human shapes. Following~\cite{unsuperise_fmap}, it is split into 80/20 for training and testing. 

\noindent\textbf{SCAPE\_r:} The remeshed version~\cite{BCICP} of SCAPE dataset\cite{scape} contains 71 human shapes. Following~\cite{unsuperise_fmap}, it is split into 51/20 for training and testing. 

\noindent\textbf{SHREC19\_r:} The remeshed version of SHREC19 dataset~\cite{melzi2019shrec} collects 44 human shapes from 11 independent datasets with distinctive poses and styles. We abandon shape $40$ due to its partiality, we test on $407$ pairs among the rest $43$ shapes, which come with ground-truth.

\noindent\textbf{DT4D-H~\cite{3dvdata}:} The remeshed subset of the large scale animation dataset DeformingThings4D~\cite{li20214dcomplete}. 
In particular, DT4D-H includes 10 categories of humanoid shapes undergoing significant pose and style variances, forming a challenging benchmark. 

\noindent\textbf{SMAL\_r:} The remeshed SMAL dataset~\cite{zuffi20173d} contains 49 animal shapes with 8 species. We follow the setting from~\cite{li2022attentivefmaps}, which splits 29 (5 species) and 20 (3 species) shapes for training and testing. 

\noindent\textbf{TOSCA\_r:} The remeshed TOSCA dataset~\cite{bronstein2008numerical} contains multiple shape categories. We choose $4$ animal categories, including \textit{cat, dog, horse}, and \textit{wolf} to verify the generalization performance of networks trained on SMAL\_r. Note that we only infer the intra-category maps, due to the absence of ground-truth inter-category maps. 

We refer readers to Supp. Material for visualizations illustrating the variability of the above datasets. 

\subsection{Near-isometric Shape Matching}\label{sec:iso}
In this part, we perform comparisons with an array of non-rigid shape matching methods: (1) Axiomatic methods including ZoomOut~\cite{zoomout}, BCICP~\cite{BCICP}, IsoMuSh~\cite{gao2021isometric}, Smooth Shells~\cite{smoothshells}, CZO~\cite{huang2020consistent}; (2) Supervised learning methods including 
TransMatch~\cite{trappolini2021shape}, GeomFMaps~\cite{donati20}, and supervised version of AttentiveFMaps~\cite{li2022attentivefmaps}; (3) Unsupervised learning methods including 
NeuroMorph~\cite{eisenberger2021neuromorph}, SyNoRiM~\cite{multi}, Deep Shell~\cite{eisenberger2020deep}, AttentiveFMaps~\cite{li2022attentivefmaps}, UDMSM~\cite{cao2022}, DUO-FM~\cite{donati2022deep}.

For all learning-based methods, we train models on \textbf{FAUST\_r} and \textbf{SCAPE\_r} respectively. Tab.~\ref{table:1} reports results on both standard tests and more challenging generalizations. 
We observe a trade-off between the two tasks, methods that performs the best in the former (e.g., supervised AttentiveFMaps and UDMSM) tend to overfit, and therefore suffer poor generalization (especially to \textbf{SREHC19\_r}). Meanwhile, our default setting, denoted by \textbf{Ours}, achieves reasonable performance in the standard tests but also outperforms the external baselines in 3 out of 4 generalization tests. Especially, in generalizing to \textbf{SHREC19\_r}, \textbf{Ours} outperforms the \emph{external} baselines by a large margin, resulting in $41\%$ (3.8 vs. 6.4) and $46\%$ (4.5 vs. 8.4) error reduction upon the second best.

We highlight that post-processing with cycle consistency generally depends on the initialized map quality and the size of the test set (e.g., $\geq 3$ shapes). In contrast, we leverage cycle consistency to improve the feature extractor during training. We also report the results from post-processing techniques based on cycle consistency~\cite{gao2021isometric, huang2020consistent, multi} in Tab.~\ref{table:1}. They are significantly outperformed by our method, which
is inferred per-pair and without any post-processing.

We further augment the dimension of functional maps in network training to $80$ (same as UDMSM), which is beneficial to near-isometric matching~\cite{li2022attentivefmaps}. It is evident that \textbf{Ours (80 dim)} achieves on-par performance with the regarding state-of-the-art methods in standard tests. On the other hand, due to the significant variability between \textbf{SHREC19\_r} and the training sets (see Supp. Material), augmenting dimension leads to worse generalization than before (5.5 vs. 3.8, 5.8 vs. 4.5). Nevertheless, even in this case, our method outperforms the \emph{external} baselines in \emph{all} generalization tests by a notable margin.

\begin{table}[]
\caption{Mean geodesic errors (×100) on SMAL\_r. The \textbf{best} and the \textbf{\textcolor{blue}{second best}} are highlighted correspondingly. }\label{table:2}
\centering
\begin{tabular}{cccc}
\hline
\rowcolor[HTML]{FFFFFF} 
\cellcolor[HTML]{FFFFFF}                         & Train                                      & \multicolumn{2}{c}{\cellcolor[HTML]{FFFFFF}SMAL\_r}              \\ \cline{2-4} 
\rowcolor[HTML]{FFFFFF} 
\multirow{-2}{*}{\cellcolor[HTML]{FFFFFF}Method} & Test                                       & SMAL\_r                             & TOSCA\_r                               \\ \hline
\rowcolor[HTML]{FFFFFF} 
\rowcolor[HTML]{FFFFFF} 
\rowcolor[HTML]{FFFFFF} 
DeepShell\cite{eisenberger2020deep}                                         & \cellcolor[HTML]{FFFFFF}                   & 29.3                                & {\textbf{\textcolor{blue}{8.7}}} \\
\rowcolor[HTML]{FFFFFF} 
GeomFMaps\cite{donati20}                                               & \multirow{-4}{*}{\cellcolor[HTML]{FFFFFF}} & 7.6                                & 24.5                                \\
\rowcolor[HTML]{FFFFFF} 
AttentiveFMaps\cite{li2022attentivefmaps}                                   & \cellcolor[HTML]{FFFFFF}                   & \textbf{5.4} & 20.9                                \\
\rowcolor[HTML]{FFFFFF} 
UDMSM\cite{cao2022}                                            & \cellcolor[HTML]{FFFFFF}                   & 24.6                                & 21.7                                \\
\rowcolor[HTML]{FFFFFF} 
DUO-FM\cite{donati2022deep}                                              & \multirow{-4}{*}{\cellcolor[HTML]{FFFFFF}} & 32.8                                & 15.3                                \\
\rowcolor[HTML]{E7E6E6} 
\textbf{Ours}                                    & \textbf{}                                  & \textbf{5.4}                        & \textbf{7.9}    \\ \hline                   
\end{tabular}
\end{table}
\begin{table*}[htbp]
\caption{Mean geodesic errors (×100) on DT4D-H. The \textbf{best} and the \textbf{\textcolor{blue}{second best}} are highlighted correspondingly. }\label{table:3}
\centering
\setlength{\tabcolsep}{15pt}
\resizebox{\textwidth}{14mm}{
\begin{tabular}{ccccc|ccc}
\hline
\rowcolor[HTML]{FFFFFF} 
\cellcolor[HTML]{FFFFFF}                         & Train                                      & \multicolumn{3}{c}{\cellcolor[HTML]{FFFFFF}DT4D-H (168)}                                                         & \multicolumn{3}{c}{\cellcolor[HTML]{FFFFFF}DT4D-H (80)}                                                          \\ \cline{3-5} \cline{6-8} 
\rowcolor[HTML]{FFFFFF} 
\multirow{-2}{*}{\cellcolor[HTML]{FFFFFF}Method} & Test                                       & DT4D-H                               & FAUST\_r                            & SCAPE\_r                                                                  & DT4D-H                               & FAUST\_r                            & SCAPE\_r                            \\ \hline
\rowcolor[HTML]{FFFFFF} 
DeepShell\cite{eisenberger2020deep}                                         & \cellcolor[HTML]{FFFFFF}                   & 27.0                                 & 4.9                                 & 6.5                                                                  & 29.3                                 & 4.7                                 & 7.0                                 \\
\rowcolor[HTML]{FFFFFF} 
AttentiveFMaps\cite{li2022attentivefmaps}                                   & \multirow{-2}{*}{\cellcolor[HTML]{FFFFFF}} & 25.7                                 & \textbf{\textcolor{blue}{3.4}} & \textbf{\textcolor{blue}{6.4}}                                      & 28.9                                 & \textbf{\textcolor{blue}{2.7}} & \textbf{\textcolor{blue}{6.3}} \\
\rowcolor[HTML]{FFFFFF} 
UDMSM\cite{cao2022}                                            &                                            & 46.8                                 & 43.3                                & 47.9                                                                    & 49.7                                 & 42.5                                & 40.0                                \\
\rowcolor[HTML]{FFFFFF} 
DUO-FM\cite{donati2022deep}                                              &                                            & \textbf{\textcolor{blue}{22.4}} & {\color[HTML]{000000} 10.0}         & {\color[HTML]{000000} 12.2}                                                & \textbf{\textcolor{blue}{24.7}} & {\color[HTML]{000000} 8.0}          & {\color[HTML]{000000} 9.2}          \\
\rowcolor[HTML]{E7E6E6} 
\textbf{Ours}                                    & \textbf{}                                  & \textbf{7.7}                       &\textbf{3.1}                        & \textbf{6.1}                                                               & \textbf{9.0}                         & \textbf{2.6}                        & \textbf{6.2}                        \\ \hline
\end{tabular}}
\vspace{-0.5em}
\end{table*}

\subsection{Non-isometric Shape Matching}\label{sec:noniso}

We also train our network on non-isometric datasets, \textbf{SMAL\_r} and \textbf{DT4D-H}, and compare it with the state-of-the-art baselines including DeepShells~\cite{eisenberger2020deep}, AttentiveFMaps~\cite{li2022attentivefmaps}, UDMSM~\cite{cao2022} and DUO-FM~\cite{donati2022deep}. 
 
\vspace{0.5em}
\noindent\textbf{SMAL\_r}:
We follow the split and input descriptors from~\cite{li2022attentivefmaps} (more details are provided in the Supp. Material). Tab.~\ref{table:2} reports results on \textbf{SMAL\_r}, our method achieves the best performance, which is on-par with AttentiveFMaps~\cite{li2022attentivefmaps}. To evaluate generalization performance, we use the trained models to directly infer intra-category maps within \textbf{TOSCA\_r}. 
It turns out that AttentiveFMaps and GeomFMaps both suffer from significant performance drops ($\times 3.8$ and $\times 3.2$ larger geodesic error). 
It is also worth noting that DeepShells achieves the second-best generalization score in the relatively simpler task. However, it fails dramatically regarding the base task. 

In contrast, our method achieves the best balance between learning in difficult non-isometric pairs and generalizing to relatively easy near-isometric pairs. 

\vspace{0.5em}
\noindent\textbf{DT4D-H}: We follow the train/test (198/95) split of~\cite{li2022attentivefmaps}, but ignore the categories \emph{mousey} and \emph{ortiz} in both train and test, due to the lack of inter-category map labels regarding them, resulting a split of 168/80. We emphasize that we conduct training and test in a \emph{category-agnostic} manner, i.e., no class label is used, and the training pairs can consist of shapes from \emph{arbitrary} two categories. This is significantly different from~\cite{li2022attentivefmaps}, in which training pairs are selected according to clustering information. Obviously, our setting is more practical, but also more challenging. For completeness, we report results under the setting of~\cite{li2022attentivefmaps} in Supp. Material and our method outperforms the baselines in both intra- and inter-category evaluation by a notable margin. 

Tab.~\ref{table:3} reports results on \textbf{DT4D-H}, in which we preserve $80$ shapes for test and train networks with $168$ and $80$ shapes, respectively. 
Note that we report mean geodesic errors over \emph{all} possible test shape pairs, which may undergo significant distortions (see, e.g., Fig.~\ref{fig:teaser}). Our method obtains a $67.2\% (7.7 \mbox{vs.} 22.4)$ geodesic error reduction with respect to the second-best baseline. 
On top of that, we also test the generalization of the trained model on near-isometric benchmarks -- our method also generalizes the best in generalization to \textbf{FAUST\_r} and \textbf{SCAPE\_r}.
The same pattern is observed when the training set is reduced by more than half. Remarkably, our network trained on the reduced set still outperforms all the baselines trained on the full set.  
\begin{figure}[t!]
    \centering
    \includegraphics[width=8.5cm]{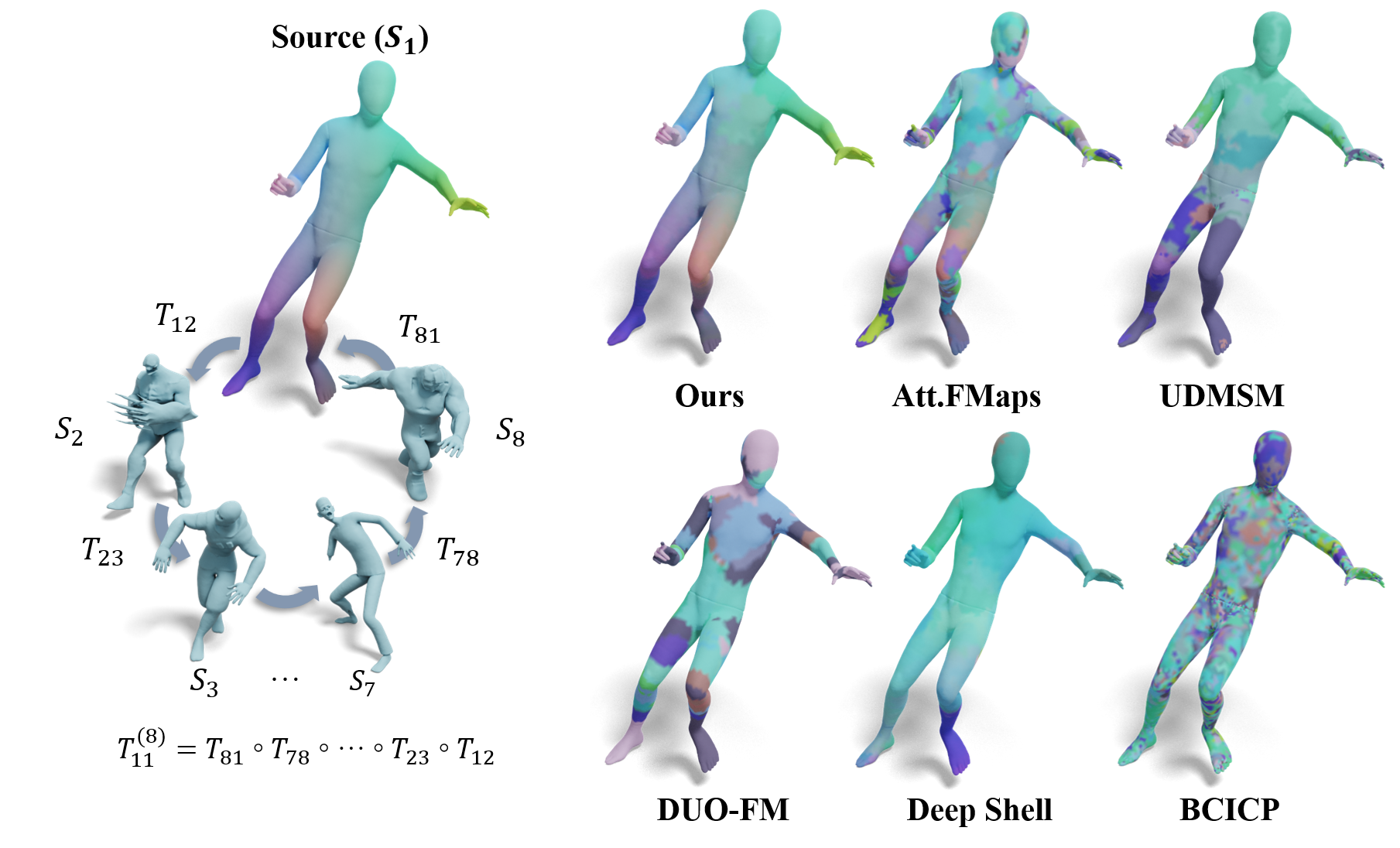}
    \caption{Qualitative evaluation of spatial cycle consistency of different methods. Even composed along a path of 8 highly deformed shapes, our resulting map remains close to identity, while all the baselines fail significantly.}\vspace{-1.5em}
    \label{fig:cycle consistency}
\end{figure}

Overall, we attribute our performance on matching challenging non-isometric shapes (Tab.~\ref{table:3}) and generalizing to unseen shapes (Tab.~\ref{table:1}) to our effort to promote both spectral and spatial cycle consistency. Especially, the isometry assumption is likely violated in the former case, thus cycle consistency, as a generic prior, plays an important role of regularizing maps. 

As an illustration, we present a qualitative evaluation on the point-wise cycle consistency in Fig.~\ref{fig:cycle consistency}. We sample $8$ shapes from the test set of DT4D-H (one from each category) and compose the maps along the path $(S_1\rightarrow S_2 \rightarrow\cdots \rightarrow S_8 \rightarrow S_1)$ with respect to different approaches. 
It is evident that due to the significant distortion undergoing among the shapes, all but our method fail to preserve cycle consistency in this demanding test, while our composing map approximates the identity map on $S_1$. It also aligns nicely with the quantitative results reported in Tab.~\ref{table:3}. 

\begin{table}[t!]
\caption{Mean geodesic errors (×100) of SURFMNet and our variant trained on $4$ datasets}\label{table:4}
\centering
\begin{tabular}{
>{\columncolor[HTML]{FFFFFF}}c 
>{\columncolor[HTML]{FFFFFF}}c 
>{\columncolor[HTML]{E7E6E6}}c }
\hline
Method   & SURFMNet & SURFMNet + Ours \\ \hline
FAUST\_r & 6.0      & \textbf{3.5}    \\
SCAPE\_r & 6.8      & \textbf{3.4}    \\
SMAL\_r  & 20.4     & \textbf{13.3}   \\
DT4D-H     & 18.3     & \textbf{15.0}   \\ \hline
\end{tabular}
\end{table}

\subsection{Integration with SURFMNet~\cite{unsuperise_fmap}}\label{sec:plugin}
Our two-branch design can be easily incorporated into any existing DFM framework following the general design outlined in Sec.~\ref{sec:dfm}. 
To demonstrate this, we modify the SURFMNet~\cite{unsuperise_fmap}, one of the earliest approaches of unsupervised DFM, by adding our new branch. 
Tab.~\ref{table:4} shows the matching accuracy on the four main benchmarks. 
It is evident that in every case, incorporating our design leads to significant error reduction ranging from $18\%$ to $50\%$ . 
Especially, in the near-isometric cases, we obtain $41.6\%$ and $50\%$ error reduction respectively. Note that the absolute scores, $0.035, 0.034$, are reasonable even compared to the state-of-the-art results reported in Tab.~\ref{table:1}. 

\subsection{Ablation Study}\label{sec:ablation}
In this section, we present a set of ablation studies consisting of two parts.  The first part verifies the rationality of our method, and the second part demonstrates the robustness of our method. We conduct all experiments on SMAL\_r dataset~\cite{zuffi20173d}. 

First of all, instead of using the updating scheme in Sec.~\ref{sec:alpha}, we test the performance of our pipeline using two fixed values of $\alpha$ in Eqn.~\eqref{eqn:pi}: $\alpha = 1$ and $\alpha = 50$. 
Compared to our proposed model, the two variants yield a noticeable performance drop. Especially in the case $\alpha = 50$, the network fails to deliver reasonable matching results. We believe it is because a large $\alpha$ amplifies the noise of maps learned at the early training stage. 

Then we justify our two-branch network design. Removing spatial branch amounts to training a standard single-branch DFM. To remove the spectral branch, we remove the \emph{FMreg} layer and instead use our new branch to compute point-wise maps and convert them to functional maps. In the end, we modify the training loss so that it covers descriptor preservation, commutativity with the Laplace-Beltrami operator, and orthogonality (the latter two compensate the removed \emph{FMreg} layer). The accuracy drop reported in the third and fourth row of Tab.~\ref{table:5} clearly suggests the necessity of our two-branch design. 

In our experiments, we always use the full resolution meshes ($\sim 5k$ vertices) and compute in Eqn.~\eqref{eqn:delta} with all of the $128$ descriptors. 
We anticipate that efficiency can become an issue when the input mesh resolution is high, and/or we would like to increase the size of learned descriptors. 
Therefore, we test the robustness of our pipeline with respect to down-sampling, which is commonly used in functional maps-based frameworks~\cite{zoomout, li2022attentivefmaps}: 
1) We down-sample $3000$ vertices on each shape via furthest point sampling; 2) In order to down-sample the feature dimension, we operate as the following during training: given a $\A_1, \A_2$, we perform SVD on $\A_1$, i.e., $\A_1 = U_1\Sigma_1 V_1^T$, then we set $\hat{\A}_1 = \A_1 \hat{V}_1$, and set $\hat{\A}_2 = \A_2 \hat{V}_1$, where $\hat{V}_1$ is the first $m$ columns of $V_1$. We set $m = 30$ by replacing $\A_i$ with $\hat{\A}_i$ in Eqn.~\eqref{eqn:delta}.
The results in the bottom two rows show that the above operation has a relatively minor effect on the performance, proving the robustness of our method. 

\begin{table}[t!]
\caption{Mean geodesic errors (×100) on different ablated settings, the models are all train on \textbf{SMAL\_r}.}\label{table:5}
\centering
\begin{tabular}{
>{\columncolor[HTML]{FFFFFF}}c 
>{\columncolor[HTML]{FFFFFF}}c }
\hline
Ours                            & \textbf{5.4}  \\ \hline
$\alpha$ = 1                         & 6.6  \\
$\alpha$ = 50                        & 35.2 \\
remove spatial branch                       & 33.4 \\
remove spectral branch                      & 14.3 \\ \hline
Vertex down-sampling    & 5.7  \\
Feature down-sampling & 5.9  \\ \hline
\end{tabular}
\end{table}

\section{Conclusion}

To conclude, we provide a theoretical justification for spectral cycle consistency in DFM. To compensate for the discrepancy of purely spectral cycle consistency, we formulate a spectral latent shape that allows the alignment of the spectral embeddings of each shape to a canonical embedding domain. Based on this construction, we introduce a two-branch architecture for estimating maps in both spectral and spatial domains within DFM. The resulting network is simple, computationally efficient, and compatible with most existing DFM frameworks. We demonstrate the effectiveness of our framework through a comprehensive set of experiments, showing significant improvements upon state-of-the-art approaches in terms of accuracy, consistency, and generalization performance.

It is worth noting, though, our method is developed upon clean and complete meshes. It would be interesting for future work to investigate the utility of our approach in more general tasks, involving partial meshes, noisy point clouds, and other representations.

\paragraph*{Acknowledgement} This work was supported in part by the National Natural Science
Foundation of China under contract No. 62171256 and Shenzhen Key Laboratory of next-generation interactive media
innovative technology (No. ZDSYS20210623092001004), and in part by the ERC Starting Grant No. 758800 (EXPROTEA) and the ANR AI Chair AIGRETTE. The authors also thank Abhishek Sharma for his input and discussion in the early development of this project.

\clearpage
In this supplementary material, we start by proving Proposition 1 in the main submission in Sec.~\ref{sec:sec1}. We then show data variability in our experiments in Sec.~\ref{sec:sec2}. Sec.~\ref{sec:sec3} clarifies the annotation preparation regarding \textbf{DT4D-H}. Finally, in Sec.~\ref{sec:sec4}, more experimental results and implementation details are provided. 

\section{Proof of Proposition 1}\label{sec:sec1}
Recall that deep functional maps are trained on $\mathcal{S}$ with respect to all possible pairs. Then the global energy is given by $E_{\mbox{total}}(\mathcal{C}) = E_{\mbox{desc}}(\mathcal{C})+E_{\mbox{reg}}(\mathcal{C}) =  \sum_{i, j} \Vert \C_{ij} \A_i - \A_j \Vert^2 + \sum_{i, j} E_{\mbox{reg}}(\C_{ij})$, where $\mathcal{C} = \{\C_{ij}\}_{i, j\in [1..n]}$ is the set of functional maps among training shapes. We restate Proposition 1 in the main submission as follows:

\begin{prop}
	\label{prop:main}
	 If $E_{\mbox{total}}(\mathcal{C}) = 0$, then for any shape $S_i$, and any path $(i, i_1, i_2, \cdots, i_k, i)$, the map composition $\C_{ii}$ is cycle consistent within the functional space spanned by the columns of $\A_i$, i.e., $\C_{ii} \A_i = \A_i$. 
\end{prop} 

\begin{proof}
It is obvious that $E_{\mbox{total}}(\mathcal{C}) = 0$ implies $E_{\mbox{desc}}(\mathcal{C}) = 0$. In the following, we show the case of the path of length 3 -- $(i, j, k, i)$. The general case follows easily. Setting $\C_{ii} = \C_{ki} \C_{jk} \C_{ij}$, we get:
\begin{equation}\label{eqn:proof}
	\C_{ii} \A_i = \C_{ki} \C_{jk} \C_{ij} \A_i = \C_{ki} \C_{jk} \A_j = \C_{ki} \A_k = \A_i.
\end{equation}
The equities in Eqn.~\eqref{eqn:proof} follow from the fact $\Vert \C_{ij} \A_i - \A_j\Vert = 0, \forall i, j$, since $E_{\mbox{desc}}(\mathcal{C}) = 0$.
\end{proof}

\begin{figure}[t!]
    \centering
    \includegraphics[width=8.5cm]{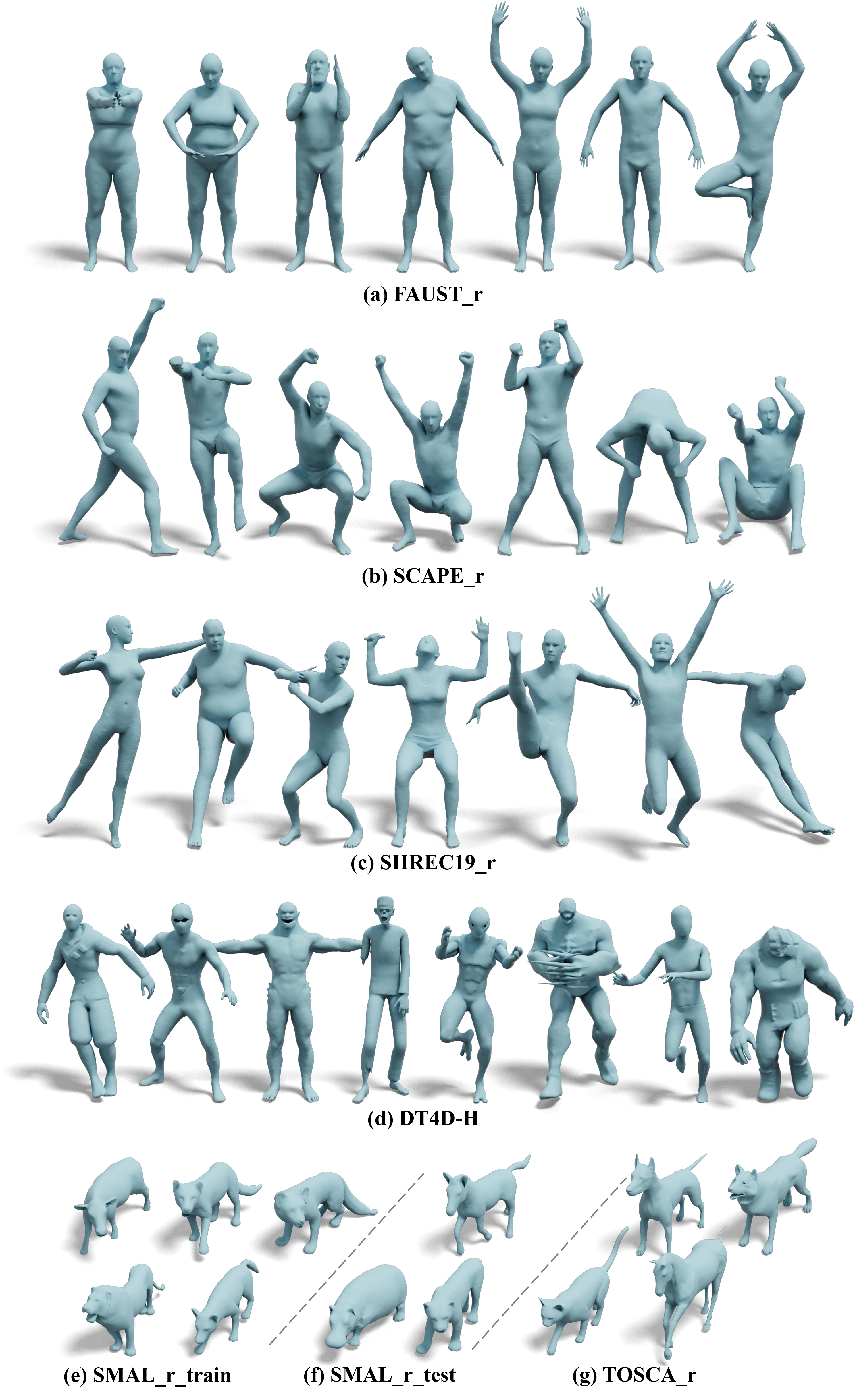}
    \caption{(a) part of the fixed poses from different individuals in \textbf{FAUST\_r}; (b) part of the different poses in \textbf{SCAPE\_r}; (c) shapes in \textbf{SHREC19\_r}; (d) 8 categories of humanoid shapes in \textbf{DT4D-H}; (e) 5 categories of animals used in training; (f) 3 categories of animals used in test; (g) test animals from \textbf{TOSCA\_r}. }
    \label{fig:dataset}
\end{figure}

\section{Data Variability}\label{sec:sec2}
In the main submission, we highlight our generalization performance. To give a hint of the distinctiveness among the involved datasets, we visualize a subset of each of them in Fig.~\ref{fig:dataset}. 
The first four rows show shapes from the humanoid datasets. \textbf{FAUST\_r} (a) consists of 10 different people with 10 fixed poses. \textbf{SCAPE\_r} (b) shows more significant pose variability but is of the same character. It is worth noting that, \textbf{SHREC19\_r} (c) manifests larger variability in \emph{both} styles and poses when compared to the above two. 
Furthermore, \textbf{DT4D-H} (d) is a new challenging dataset consisting of distinctive humanoid categories, in which the inter-class maps are highly non-isometric, especially when compared to the aforementioned datasets. 

There are $8$ species of animals in \textbf{SMAL\_r}. Following~\cite{li2022attentivefmaps}, we use 5 of them during training and the rest for testing. As shown in Fig.~\ref{fig:dataset} (e) and (f), we observe obvious differences between them, rendering the difficulty of the task. In addition, the $31$ animal shapes from \textbf{TOSCA\_r} (g) fall into 4 categories and also demonstrate noticeable differences from the training set of \textbf{SMAL\_r}.

\section{Label Preparation in \textbf{DT4D-H}}\label{sec:sec3}

Note that the inter-class correspondence annotations from \textbf{DT4D-H} are only available between category \emph{crypto} and the other 7 categories. 
In order to train and test on this benchmark in a \emph{category-agnostic} manner, we compute an inter map between two shapes, $S_1, S_2$, from  two categories other than \emph{crypto}, with the following composition: 
\[T_{12} = T_{c2} \circ T_{1c},\]
where $T_{c2}, T_{c1}$ are the annotated inter-class maps regarding the center category, \emph{crypto}. Note again, we exclude categories \emph{mousey} and \emph{ortiz} in the experimental setting reported in the main submission, simply due to their lack of inter-class correspondence annotation with respect to the center category.

However, empirically we observe that certain noise in the original annotation is amplified through the above composition, leading to a small portion of erroneous labels. 
To alleviate such discrepancy for better evaluation, we filter the composed correspondences as follows: Given composed maps $T_{12}, T_{21}$, we further compose them to obtain self maps on $S_1$ and $S_2$, respectively. That is, $T_{11} = T_{21}\circ T_{12}, T_{22} = T_{12}\circ T_{21}$. Then, we evaluate per-vertex Euclidean errors of the self-maps with respect to the ground truth identity maps. Finally, we filter out all the annotated correspondences involving vertices such that $\Vert T_{ii} (p) - p\Vert > 0.1$ (all shapes are normalized to unit total area). 

In Fig.~\ref{fig:label}, we visualize the correspondences before and after our post-processing. We remark that the ground-truth annotations are not dense. That is, there exists a portion of vertices on one shape corresponding to no vertex on the other, which is indicated by the black color in the transferred texture. As illustrated within the circles of the zoom-in regions, our post-processing manages to remove the wrongly mapped points (see the discontinuous purple regions at the top). As a result, the removed region is now in no correspondence (see the black region at the bottom). On average, about $1\%$ of the points in the original annotation are filtered out.

\begin{figure}[t!]
    \centering
    \includegraphics[width=8cm]{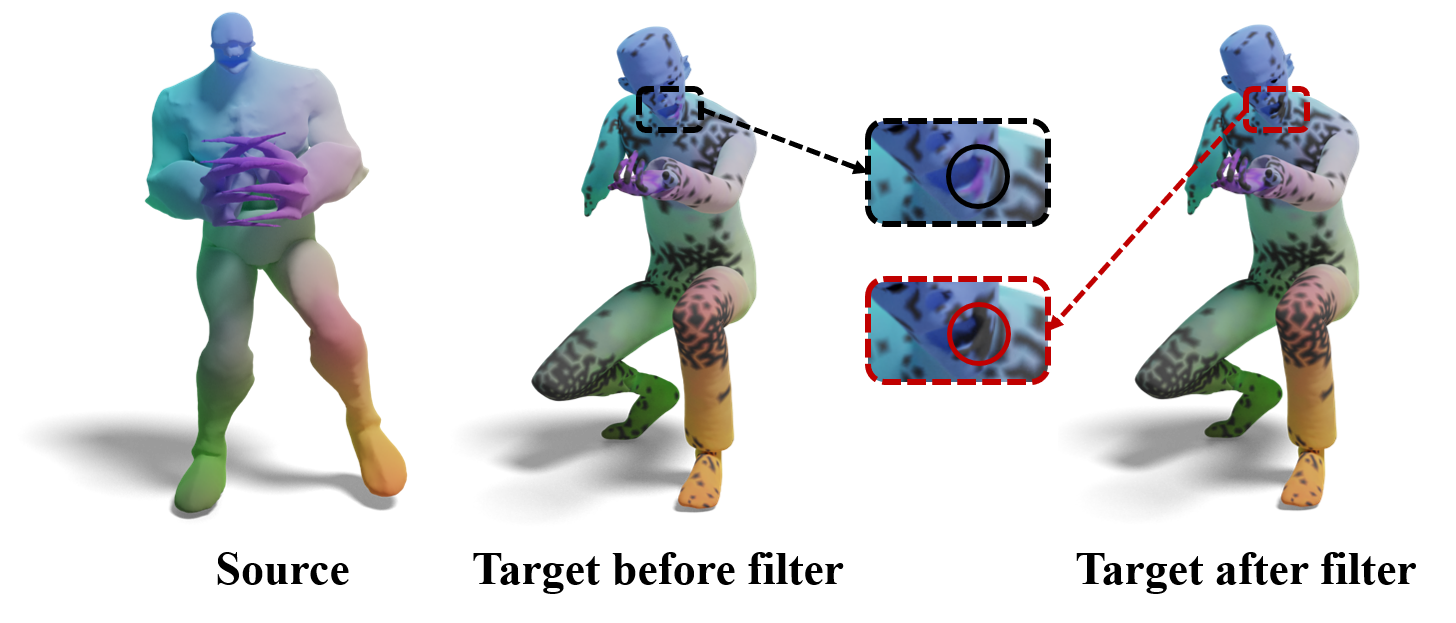}
    \caption{We filter out the erroneous correspondences via consistency prior. See the text for details.}
    \label{fig:label}
\end{figure}

\section{More Experimental Results and Details}\label{sec:sec4}
In this section, we provide not only implementation details, but also more experimental results, both quantitatively and qualitatively, to further clarify and support our claims made in the main submission. 

\subsection{Single Pair Fine-tune}
In this section, we perform a challenging fine-tuning task to test our method and two state-of-the-art unsupervised  methods -- AttentiveFMaps~\cite{li2022attentivefmaps} and UDMSM~\cite{cao2022}. 

We select two pairs of non-isometric animals from \textbf{TOSCA\_r}. Then we use the weights trained on the \textbf{human} dataset \textbf{FAUST\_r} from our experiments as initialization and perform fine-tuning on the selected \textbf{animal} pairs. 
All methods are optimized for $100$ epochs over the given pair. 
The qualitative comparisons are shown in Fig.~\ref{fig:finetune}. Note that our method is the only one that leads to good maps, by which the grid texture (e.g., on the torsos) is well-preserved.

\begin{figure}[t!]
    \centering
    \includegraphics[width=8.5cm]{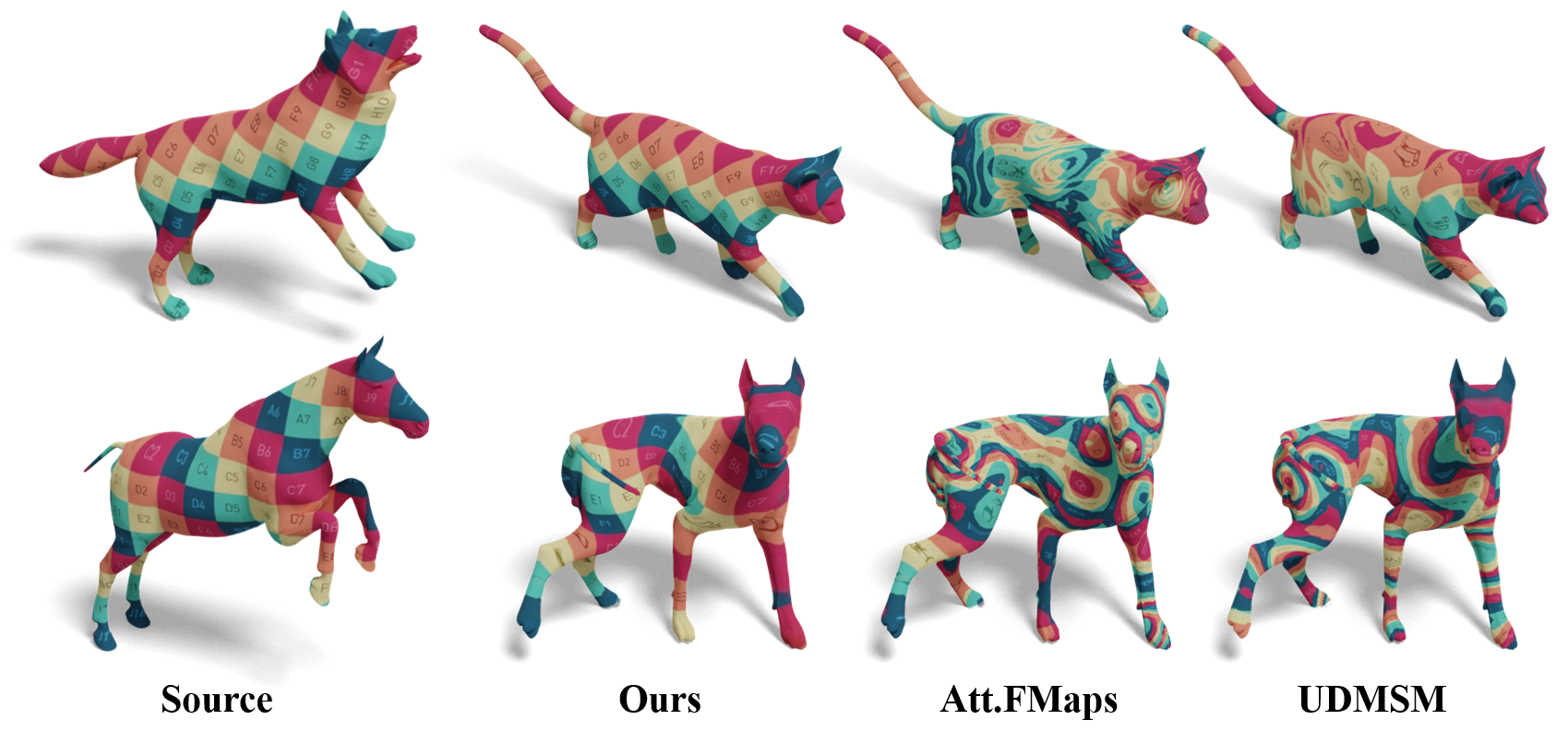}
    \caption{We perform unsupervised fine-tune on two pairs of non-isometric animals with weights initialized from models trained on \textbf{FAUST\_r}. Our results clearly outperform the competing methods. }
    \label{fig:finetune}
\end{figure}

\begin{table*}[htbp]
\caption{Mean geodesic errors (×100) on \textbf{FAUST\_r}, 
\textbf{SCAPE\_r}, and \textbf{SHREC19\_r}. The \textbf{best} and the \textbf{\textcolor{blue}{second best}} are highlighted correspondingly. }\label{table:5}
\centering
\setlength{\tabcolsep}{11pt}
\resizebox{\textwidth}{48mm}{
\begin{tabular}{ccccccccc}
\hline
\rowcolor[HTML]{FFFFFF} 
\cellcolor[HTML]{FFFFFF}                         & Train                                                  & \multicolumn{3}{c}{\cellcolor[HTML]{FFFFFF}FAUST\_r}                                                            & \multicolumn{1}{l}{\cellcolor[HTML]{FFFFFF}} & \multicolumn{3}{c}{\cellcolor[HTML]{FFFFFF}SCAPE\_r}                                                            \\ \cline{3-5} \cline{7-9} 
\rowcolor[HTML]{FFFFFF} 
\multirow{-2}{*}{\cellcolor[HTML]{FFFFFF}Method} & Test                                                   & FAUST\_r                            & SCAPE\_r                            & SHREC19\_r                             &                                              & SCAPE\_r                            & FAUST\_r                            & SHREC19\_r                             \\ \hline
\rowcolor[HTML]{FFFFFF} 
ZM\cite{zoomout}                                                &                                                        & 6.1                                 & \textbackslash{}                    & \textbackslash{}                    &                                              & 7.5                                 & \textbackslash{}                    & \textbackslash{}                    \\
\rowcolor[HTML]{FFFFFF} 
BCICP\cite{BCICP}                                             &                                                        & 6.4                                 & \textbackslash{}                    & \textbackslash{}                    &                                              & 11.0                                & \textbackslash{}                    & \textbackslash{}                    \\
\rowcolor[HTML]{FFFFFF} 
IsoMuSh\cite{gao2021isometric}                                             &                                                        & 4.4                                 & \textbackslash{}                    & \textbackslash{}                    &                                              & 5.6                                & \textbackslash{}                    & \textbackslash{}                    \\

\rowcolor[HTML]{FFFFFF} 
Smooth Shell\cite{smoothshells}                                      &                                                        & 2.5                                 & \textbackslash{}                    & \textbackslash{}                    &                                              & 4.7                                 & \textbackslash{}                    & \textbackslash{}                    \\
\rowcolor[HTML]{FFFFFF} 
CZO\cite{huang2020consistent}                                             &                                                        & 2.2                                 & \textbackslash{}                    & \textbackslash{}                    &                                              & 2.5                                & \textbackslash{}                    & \textbackslash{}                    \\ \hline
\rowcolor[HTML]{F5F5F5} 
FMNet\cite{litany2017deep}                                             & \cellcolor[HTML]{F5F5F5}                                & 11.0                                & 30.0                                & \textbackslash{}                    &                                              & 17.0                                & 33.0                                & \textbackslash{}                    \\
\rowcolor[HTML]{F5F5F5} 
3D-CODED\cite{groueix20183d}                                          & \cellcolor[HTML]{F5F5F5}                               & 2.5                                 & 31.0                                & \textbackslash{}                    &                                              & 31.0                                & 33.0                                & \textbackslash{}                    \\
\rowcolor[HTML]{F5F5F5} 
HSN\cite{wiersma2020cnns}                                              & \cellcolor[HTML]{F5F5F5}                                & 3.3                                 & 25.4                                & \textbackslash{}                    &                                              & 3.5                                 & 16.7                                & \textbackslash{}                    \\
\rowcolor[HTML]{F5F5F5} 
ACSCNN\cite{li2020shape}                                           & \cellcolor[HTML]{F5F5F5}                                & 2.7                                 & 8.4                                 & \textbackslash{}                    &                                              & 3.2                                 & 6.0                                 & \textbackslash{}                    \\
\rowcolor[HTML]{FFFFFF} 
TransMatch\cite{trappolini2021shape}                                        & \cellcolor[HTML]{FFFFFF}                               & 2.7                                 & 33.6                                & 21.0                                &                                              & 18.3                                & 18.6                                & 38.8                                \\
\rowcolor[HTML]{FFFFFF} 
GeomFMaps\cite{donati20}                                        & \cellcolor[HTML]{FFFFFF}                               & \textbf{\textcolor{blue}{2.6}}                                 & \textbf{\textcolor{blue}{3.3}}                                & \textbf{\textcolor{blue}{9.9}}                                &                                              & \textbf{\textcolor{blue}{3.0}}                                & \textbf{\textcolor{blue}{3.0}}                                & \textbf{12.2}                                \\
\rowcolor[HTML]{FFFFFF} 
AttentiveFMaps\cite{li2022attentivefmaps}                                            & \multirow{-7}{*}{{\cellcolor[HTML]{FFFFFF}sup}}   & \textbf{1.4 }                               & \textbf{2.2}                                 & \textbf{9.4}                                 &                                              & \textbf{1.7 }                                & \textbf{1.8} & \textbf{12.2}                                \\ \hline
\rowcolor[HTML]{F5F5F5} 
SURFMNet\cite{unsuperise_fmap}                                         & \cellcolor[HTML]{F5F5F5}                                 & 6.0                                & 16.5                                & \textbackslash{}                    &                                              & 6.8                                & 18.5                                & \textbackslash{}                    \\
\rowcolor[HTML]{F5F5F5} 
UnsupFMNet\cite{halimi2018self}                                        & \cellcolor[HTML]{F5F5F5}                                 & 10.0                                & 29.0                                & \textbackslash{}                    &                                              & 16.0                                & 22.0                                & \textbackslash{}                    \\
\rowcolor[HTML]{F5F5F5} 
WSupFMNet\cite{sharma20}                                        & \cellcolor[HTML]{F5F5F5}                                 & 3.3                                 & 11.7                                & \textbackslash{}                    &                                              & 7.3                                 & 6.2                                 & \textbackslash{}                    \\
\rowcolor[HTML]{FFFFFF} 
NeuroMorph\cite{eisenberger2021neuromorph}                                       & \cellcolor[HTML]{FFFFFF}                               & 8.5                                 & 28.5                                & 26.3                                &                                              & 29.9                                & 18.2                                & 27.6                                \\
\rowcolor[HTML]{FFFFFF} 
SyNoRiM\cite{multi}                                       & \cellcolor[HTML]{FFFFFF}                               & 7.9                                 & 21.7                                & 25.5                                &                                              & 9.5                                & 24.6                                & 26.8                                \\
\rowcolor[HTML]{FFFFFF} 
Deep Shell\cite{eisenberger2020deep}                                        & \cellcolor[HTML]{FFFFFF}                               & \textbf{\textcolor{blue}{1.7}} & 5.4                                 & 27.4                                &                                              & 2.5                                 & 2.7                                 & 23.4                                \\
\rowcolor[HTML]{FFFFFF} 
AttentiveFMaps\cite{li2022attentivefmaps}                                   & \cellcolor[HTML]{FFFFFF}                               & 1.9                                 & \textbf{2.6} & 6.4 &                                              & \textbf{\textcolor{blue}{2.2}} & \textbf{\textcolor{blue}{2.2}}                        & 9.9                                 \\
\rowcolor[HTML]{FFFFFF} 
UDMSM\cite{cao2022}                                            & \cellcolor[HTML]{FFFFFF}                               & \textbf{1.5}                        & 7.3                                 & 21.5                                & \textbf{}                                    & \textbf{2.0}                        & 8.6                                 & 30.7                                \\
\rowcolor[HTML]{FFFFFF} 
DUO-FM\cite{donati2022deep}                                               & \cellcolor[HTML]{FFFFFF} & 2.5                                 & 4.2                                 & 6.4                                 &                                              & 2.7                                 & 2.8                                 & 8.4 \\
\rowcolor[HTML]{E7E6E6} 
\textbf{Ours}                                          & \multirow{-8}{*}{{\cellcolor[HTML]{E7E6E6}unsup}}                               & 2.3                        & \textbf{2.6}                                 & \textbf{3.8}                                 & \textbf{}                                    & 2.4                        & 2.5                                 & \textbf{4.5}                               \\
\rowcolor[HTML]{E7E6E6}  
\textbf{Ours (80 dim)}                                    & \cellcolor[HTML]{E7E6E6}                                              & {\textbf{\textcolor{blue}{1.7}}}                                 & \textbf{2.6}                        & {\textbf{\textcolor{blue}{5.5}}}                        &                                              & {\textbf{\textcolor{blue}{2.2}}}                                 & \textbf{2.0} & {\textbf{\textcolor{blue}{5.8}}}                         \\ \hline
\end{tabular}\vspace{-0.5em}}
\end{table*}

\subsection{Additional Baselines on Near-isometric Datasets}
Due to the space limit, we only present the more recent and stronger baselines in Tab.~\ref{table:1} in the main submission. In Tab.~\ref{table:5}, we provide more complete results on near-isometric shape matching. Note that the newly introduced baselines (highlighted in light gray) are in general weaker than the baselines we report in the main submission, therefore their absence does not affect our experimental analysis. 

\begin{figure}[t]
    \centering
    \includegraphics[width=8.5cm]{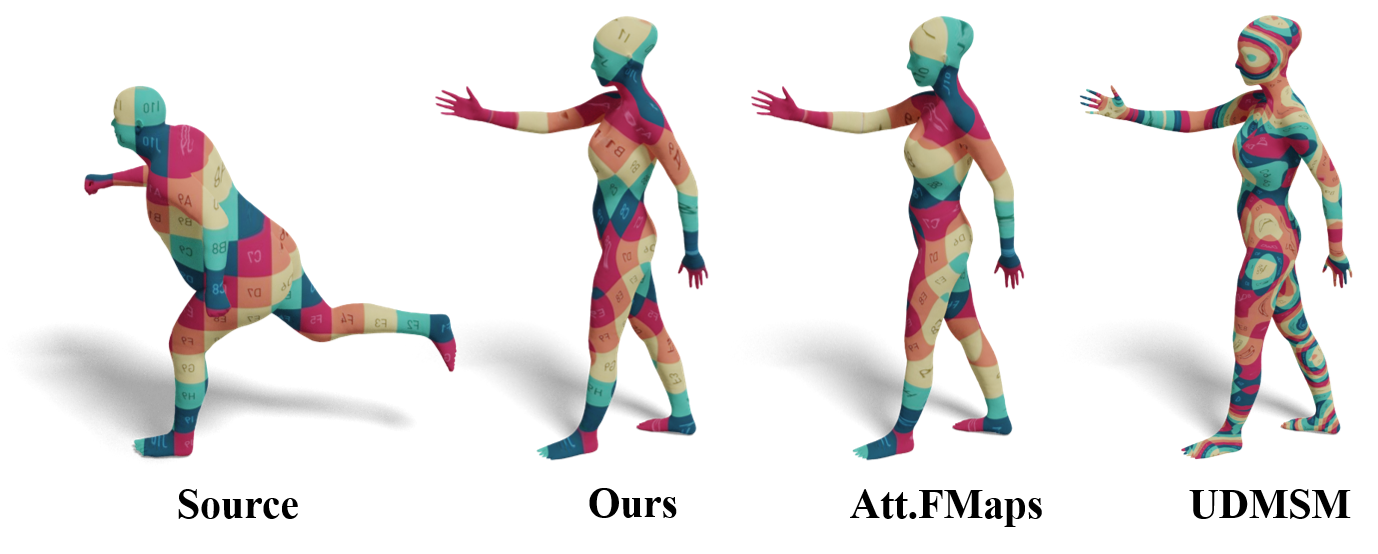}
    \caption{We train models on \textbf{FAUST\_r} and test on \textbf{SHREC19\_r}. }
    \label{fig:sup1}
\end{figure}

In Fig.~\ref{fig:sup1}, we provide qualitative results to demonstrate the generalization power of our method. Specifically, we train models on \textbf{FAUST\_r} and infer a pair of shapes from \textbf{SHREC19\_r}. The qualitative results are consistent with the quantitative results in Tab.~\ref{table:5}.

\subsection{Implementation Details on SMAL\_r} 
In this part, we clarify our experiments setting of \textbf{SMAL\_r} (see Tab.~\ref{table:2} in the main submission). We follow the setting of \cite{li2022attentivefmaps}, where the training and testing data contain 5 and 3 species, respectively. Also following \cite{li2022attentivefmaps}, we use the XYZ signal input augmented with random rotations around the up (or Y) axis as the input signal to the network. The same settings are applied to the baseline GeomFMaps \cite{donati20}. For UDMSM \cite{cao2022} and DeepShell \cite{eisenberger2020deep}, we have implemented the official codes by the regarding authors with \emph{both} SHOT ~\cite{shot} (the common default descriptors) and XYZ as input. And in the end, we select the better output from the two. In fact, both methods work better with SHOT as input. 

In Fig.~\ref{fig:sup2}, we train models on \textbf{SMAL\_r} and test on \textbf{TOSCA\_r}. The qualitative results suggest our better generalization performance, which agrees with the quantitative results reported in Tab.~\ref{table:2} in the main submission.  

\begin{figure}[t]
    \centering
    \includegraphics[width=8.5cm]{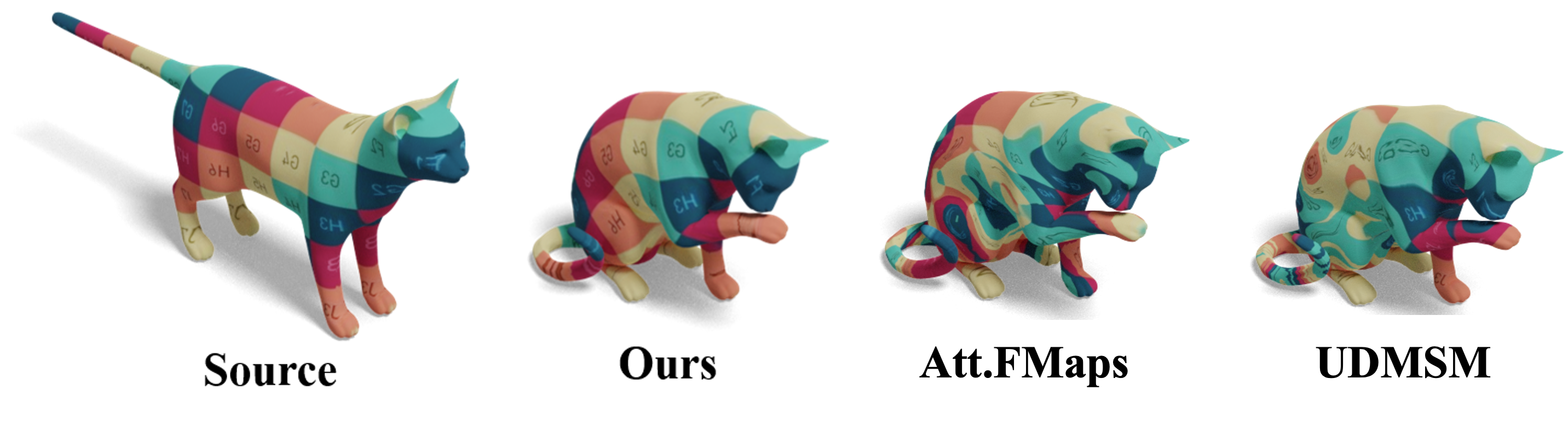}
    \caption{We train models on \textbf{SMAL\_r} and test on \textbf{TOSCA\_r}.}
    \label{fig:sup2}
\end{figure}

\subsection{Implementation details on DT4D-H}

We find in the official code of \cite{li2022attentivefmaps} that the authors \emph{train} and test inter maps with a fixed source category (\emph{crypto}). On the other hand, in the main submission, we advocate a category-agnostic \emph{training} scheme, which is more practical as well as challenging (see Tab.~\ref{table:3} in the main submission for comparison). For the sake of completeness and fairness, we follow the exact experimental settings of AttentiveFMaps to train our model and report the results in the same manner as \cite{li2022attentivefmaps} in Tab.~\ref{table:6}. We outperform \cite{li2022attentivefmaps} by a significant margin ($6.1$ vs. $11.6$ for inter-class maps) in their setting. Remarkably, as an unsupervised method, our inter-class score is even comparable with the baselines with supervision (see the top two rows). 

\begin{table}[t!]
\caption{Mean geodesic errors (×100) on \textbf{DT4D-H} followed AttentiveFMaps. The \textbf{best} and the \textbf{\textcolor{blue}{second best}} are highlighted correspondingly. }\label{table:6}
\centering
\begin{tabular}{cccc}
\hline
\rowcolor[HTML]{FFFFFF} 
\cellcolor[HTML]{FFFFFF}                         &                                                 & \multicolumn{2}{c}{\cellcolor[HTML]{FFFFFF}DT4D}                           \\ \cline{3-4} 
\rowcolor[HTML]{FFFFFF} 
\multirow{-2}{*}{\cellcolor[HTML]{FFFFFF}Method} &                                                 & intra-class                         & inter-class                          \\ \hline
\rowcolor[HTML]{FFFFFF} 
GeomFMaps\cite{donati20}                                          & \cellcolor[HTML]{FFFFFF}                        & 2.1                                 & 4.1                                  \\
\rowcolor[HTML]{FFFFFF} 
AttentiveFMaps\cite{li2022attentivefmaps}                                   & \multirow{-2}{*}{\cellcolor[HTML]{FFFFFF}sup}   & 1.8                                 & 4.6                                  \\ \hline
\rowcolor[HTML]{FFFFFF} 
DeepShell\cite{eisenberger2020deep}                                        & \cellcolor[HTML]{FFFFFF}                        & 3.4                                 & 31.1                                 \\
\rowcolor[HTML]{FFFFFF} 
GeomFMaps\cite{donati20}                                          & \cellcolor[HTML]{FFFFFF}                        & 3.3                                 & 22.6                                 \\
\rowcolor[HTML]{FFFFFF} 
AttentiveFMaps\cite{li2022attentivefmaps}                                   & \multirow{-3}{*}{\cellcolor[HTML]{FFFFFF}unsup} & \textbf{\textcolor{blue}{1.7}} & \textbf{\textcolor{blue}{11.6}} \\
\rowcolor[HTML]{E7E6E6} 
\textbf{Ours}                                    & \textbf{}                                       & \textbf{1.2}                        & \textbf{6.1}                         \\ \hline
\end{tabular}
\vspace{-1em}
\end{table}

\subsection{Implementation Details on Plugin with SURFMNet} We implement our two-branch variant of SURFMNet with PyTorch~\cite{pytorch}. The dimension of the Laplace-Beltrami eigenbasis is set to 40. SHOT~\cite{shot} descriptors are used as the input signal of the network. The dimensions of the input and the output descriptors are both set to 352. During training, the value of the learning rate is set to 1e-3 with ADAM optimizer. In all experiments, we set the batch size to 1. We initialize $\alpha$ to $1$ and increase it by $1$ per epoch. Note that this learning scheme is different from the one we reported in our main submission, where the backbone is DiffusionNet~\cite{diffusionNet}. Here $\alpha$ is augmented slower as the backbone network of SURFMNet~\cite{unsuperise_fmap} is weaker. We keep all the losses used in SURFMNet~\cite{unsuperise_fmap}, and just simply add our proposed new branch, as shown in Fig. \ref{fig:pipeline} in the main submission. 

\end{document}